\documentclass{article} % For LaTeX2e
\usepackage{collas2026_conference,times}
\usepackage{amsmath,amsfonts,bm}

\def\eqref#1{equation~\ref{#1}}
\def\1{\bm{1}}

\DeclareMathAlphabet{\mathsfit}{\encodingdefault}{\sfdefault}{m}{sl}
\SetMathAlphabet{\mathsfit}{bold}{\encodingdefault}{\sfdefault}{bx}{n}

\usepackage{graphicx}
\usepackage{subcaption}
\usepackage[utf8]{inputenc}
\usepackage[T1]{fontenc}
\usepackage{booktabs}
\usepackage{multirow}
\usepackage[table]{xcolor}
\usepackage{amsmath}
\usepackage{amsthm}
\newtheorem{theorem}{Theorem}

\usepackage{xcolor}
\usepackage{algorithm}
\usepackage{algpseudocode}
\usepackage{float}
\usepackage{hyperref}
\hypersetup{
    colorlinks=true,
    linkcolor=red,
    filecolor=magenta,
    urlcolor=blue,
    citecolor=purple,
    pdftitle={CP-MoE: Consistency-Preserving Mixture-of-Experts for Continual Learning},
    }

\algrenewcommand\alglinenumber[1]{{#1}}

\title{CP-MoE: Consistency-Preserving Mixture-of-Experts for Continual Learning}

\author{Antiquus S.~Hippocampus, Natalia Cerebro  \thanks{ Use footnote for providing further information
about author (webpage, alternative address)---\emph{not} for acknowledging
funding agencies.  Funding acknowledgements go at the end of the paper.} \\
Department of Computer Science\\
Random University\\
Country \\
\texttt{\{hippo,brain\}@cs.random.edu} \\
\And % Use And to have authors side by side
Koala Learnus \& D. Q. ResNet  \\
Department of Computational Neuroscience \\
University of Random City \\
Another Country \\
\texttt{\{koala,net\}@random.rand} \\
\AND % Use AND to have authors block one under the other
Coauthor \\
Affiliation \\
Address \\
\texttt{email}
}

\collasfinalcopy % Uncomment for camera-ready version, but NOT for submission.

\author{%
  Yang Liu \\
  UNSW Sydney, Australia \\
  \texttt{yang.liu39@student.} \\
  \texttt{unsw.edu.au} \\
  \And
  Toan Nguyen \\
  UNSW Sydney, Australia \\
  \texttt{toan.nguyen@unsw.edu.au} \\
  \And
  Flora D. Salim \thanks{Corresponding Author.} \\
  UNSW Sydney, Australia \\
  \texttt{flora.salim@unsw.edu.au} \\
}
\begin{document}

\maketitle

\begin{abstract}

% Continual learning remains a Central challenge for Large Language Models (LLMs) due to catastrophic forgetting. Recently, Mixture-of-Experts (MoE) architectures have emerged as a promising solution, leveraging their partitioned parameter space to mitigate parameter interference through sparse activation. However, in the context of LoRA-based MoE, existing approaches typically suffer from either excessive isolation that restricts knowledge transfer or a lack of effective parameter protection. 

Catastrophic forgetting remains a major obstacle to continual learning in large language models (LLMs) and vision--language models (VLMs). Although Mixture-of-Experts (MoE) architectures offer an efficient path to scaling, existing LoRA-based MoE continual learning methods still face a fundamental trade-off: they either isolate experts too aggressively, limiting knowledge transfer across tasks, or allow task-specific updates to overwrite important existing parameters, leading to severe forgetting. To address this, we propose \textit{CP-MoE}, a continual learning framework built around a transient expert that captures early task-specific updates and guides their integration into stable experts. CP-MoE introduces a consistency-preserving routing bias, which uses the transient expert to estimate representation similarity with stable experts and steer routing towards more compatible expert selection, and a transient expert-guided regularisation mechanism, which selectively protects important historical parameters during merging. Together, these components reduce parameter interference and forgetting while preserving cross-task knowledge transfer. We validate CP-MoE on both unimodal and multimodal continual learning benchmarks with LLM-based and VLM-based MoE models. On SuperNI benchmark, spanning diverse sequential language tasks, CP-MoE achieves state-of-the-art performance and stronger zero-shot transfer to unseen tasks. On VQA v2 dataset, it scales effectively to multimodal visual reasoning, consistently reduces forgetting, and outperforms strong MoE baselines. Our source code is available at \url{https://github.com/YangLiu-Lewis/CP-MoE}.

% To address these issues, we propose a framework centered on a \textbf{transient expert} mechanism. This transient expert processes initial task updates to evaluate representation similarity, isolating task-specific knowledge before guiding its integration into stable experts. 
% We further introduce a \textbf{consistency-preserving routing bias (CP Bias)}, where the transient expert identifies relevant and irrelevant stable experts based on representation similarity and injects a bias into the routing scores to favour consistent expert selection.
%  a \textbf{transient expert-guided parameter regularization (TE Regularization)} that selectively protects critical historical weights. Together, these mechanisms directly mitigate parameter interference and catastrophic forgetting while enabling knowledge transfer between special tasks. 

% We validate the scalability and generalizability of our method across unimodal and multimodal domains. On SuperNI language generation tasks, our approach achieves state-of-the-art performance on sequential learning while demonstrating significantly superior zero-shot transfer capabilities to unseen domains. Furthermore, our method scales effectively to the VQA v2 benchmark, consistently mitigating catastrophic forgetting across diverse modalities and confirming its robustness as a versatile continual learning framework.

\end{abstract}
%%%%%%%%%%%%%%%%%%%%%%%%%%%%%%%%%%%%%%%%%%%%%%%%%%%%%%%%%%%%%%%%%%%%%%%%%%%%%%%
\section{Introduction}
%%%%%%%%%%%%%%%%%%%%%%%%%%%%%%%%%%%%%%%%%%%%%%%%%%%%%%%%%%%%%%%%%%%%%%%%%%%%%%%
\label{sec:intro}

Large language models (LLMs) and vision--language models (VLMs) have demonstrated strong performance across a wide range of tasks~\citep{liu2024improved,liu2024deepseek,touvron2023llama}. Yet, unlike static benchmark settings, real-world deployment requires these models to continually adapt to non-stationary data streams, where tasks and data distributions evolve over time. This setting, known as continual learning (CL), requires a model to accumulate knowledge across a sequence of tasks without degrading previously acquired capabilities. Despite substantial progress, catastrophic forgetting remains a central challenge~\citep{mccloskey1989catastrophic,yu2024recent}, as updates for new tasks can significantly overwrite earlier representations.

To adapt large foundation models to new tasks efficiently, recent work has increasingly focused on sparse architectures, particularly Mixture-of-Experts (MoE)~\citep{shazeer2017outrageously}. In these models, the backbone is typically frozen to preserve core pre-trained knowledge, while adaptation is handled by a small set of trainable expert modules. The appeal of MoE lies in its partitioned parameter space, which enables task-specific computation paths and provides a natural way to reduce interference across tasks~\citep{aljundi2017expert}. When combined with parameter-efficient fine-tuning methods such as LoRA~\citep{hu2022lora}, this gives rise to LoRA-MoE frameworks~\citep{dou2023loramoe}, which have emerged as a practical and scalable approach to continual adaptation. However, despite their structural sparsity and efficiency, MoE-based models still remain vulnerable to catastrophic forgetting under sequential task training.

 % still face a fundamental trade-off: they either isolate experts too aggressively, limiting knowledge transfer across tasks, or permit task-specific updates to overwrite important existing parameters, leading to severe forgetting. To address this, we propose \textit{CP-MoE}, a continual learning framework built around a transient expert that captures early task-specific updates and guides their integration into stable experts. CP-MoE introduces a consistency-preserving routing bias, which uses the transient expert to estimate representation similarity with stable experts and steer routing towards more compatible expert selection, and a transient expert-guided regularisation mechanism, which selectively protects important historical parameters during merging. Together, these components reduce parameter interference and forgetting while preserving cross-task knowledge transfer. We validate CP-MoE on both unimodal and multimodal continual learning benchmarks with LLM-based and VLM-based MoE models. On SuperNI, spanning diverse sequential language tasks, CP-MoE achieves state-of-the-art performance and stronger zero-shot transfer to unseen tasks. On VQA v2, it scales effectively to multimodal visual reasoning, consistently reduces forgetting, and outperforms strong baselines.

Despite their recent success, continual LoRA-MoE frameworks still face key limitations. Existing methods often prevent forgetting by isolating experts too aggressively~\citep{liang2025gainlora}, typically through a learn-and-freeze strategy that adds new LoRA experts for each task and then keeps them fixed. As a result, the model behaves as a static ensemble with limited cross-task knowledge transfer, leading to weaker generalisation on unseen tasks and poor scalability as the task sequence grows. Second, the load-balancing constraint commonly used in MoE architectures for hardware efficiency can be problematic in continual learning, as it tends to enforce unnecessary uniformity across experts and can limit task-specific specialisation. This calls for routing constraints that account for the semantic alignment between task representations and experts. Finally, most current MoE-based continual learning methods lack mechanisms to protect important parameters during sequential updates, making them vulnerable to forgetting~\citep{huai2025cl}.

To address these challenges, we propose \emph{CP-MoE}, a continual learning framework built around a \emph{transient expert} that acts as a lightweight probe before updates are integrated into the stable expert pool. For each incoming task, the transient expert is briefly adapted on a small warm-up subset to capture early task-specific updates and assess its representation compatibility with existing experts. This design offers three key advantages. First, the transient expert provides a local look-ahead signal of task-induced parameter displacement, enabling more selective integration without dynamic expert expansion. Second, it supports a \emph{consistency-preserving routing bias} that favours experts with more compatible representations, mitigating the excessive uniformity imposed by standard load-balancing constraints and promoting stronger expert specialisation. Third, it enables a \emph{representation-guided regularisation} mechanism that selectively protects important historical parameters during integration, thereby reducing interference and forgetting. Together, these components allow CP-MoE to preserve prior knowledge while maintaining effective knowledge transfer across related tasks. We show that this idea is highly effective in practice. Although CP-MoE introduces only a transient expert for each task, this module is used only during the short warm-up stage and is not retained thereafter, allowing the framework to improve performance without incurring persistent computational overhead. We evaluate CP-MoE on both unimodal and multimodal continual learning benchmarks with LLM-based and VLM-based MoE models, where it achieves state-of-the-art performance and strong zero-shot generalisation to out-of-distribution tasks.

Our contributions are summarised as follows:
\begin{itemize}
    \item \textbf{Transient expert for selective consolidation.} We introduce a transient expert that briefly adapts to each incoming task and serves as a lightweight probe of early task-specific updates and representation compatibility, enabling selective integration into a fixed pool of stable experts without dynamic expansion.

    \item \textbf{Consistency-preserving routing bias.} We propose a routing mechanism that favours experts with more compatible representations, improving routing stability and promoting stronger expert specialisation beyond standard load balancing.

    \item \textbf{Representation-guided regularisation.} We develop a dynamic regularisation scheme that selectively protects important historical parameters during integration, thereby reducing interference and forgetting.

    \item \textbf{Strong empirical performance.} We validate CP-MoE on SuperNI and VQA v2 with LLM-based and VLM-based methods. CP-MoE achieves state-of-the-art performance, stronger zero-shot transfer on unseen tasks, and consistently lower forgetting than strong baselines.
    
\end{itemize}
\section{Related Work}

\paragraph{Continual Learning.}
Continual learning (CL) methods for mitigating catastrophic forgetting are commonly grouped into three paradigms: \textbf{regularisation-based} methods, which constrain updates to parameters deemed important for previous tasks~\citep{kirkpatrick2017overcoming,zenke2017continual}; \textbf{memory-based} methods, which replay stored samples to preserve past decision boundaries~\citep{lopezpaz2017gradient,chaudhry2019efficient,nguyen2024class,buzzega2020dark}; and \textbf{expansion-based} methods, which allocate task-specific parameters or network branches as new tasks arrive~\citep{rusu2016progressive,von2019continual}. While effective for smaller models, these approaches scale poorly to large language models (LLMs) and vision--language models (VLMs), where operating over the full parameter space is often prohibitively expensive in memory and computation. Moreover, many existing CL methods were developed primarily for image classification and do not transfer naturally to large-scale multimodal settings.

\paragraph{Parameter-Efficient Fine-Tuning in Continual Learning.}
Recent work has shown that parameter-efficient fine-tuning (PEFT) can improve continual learning performance by adapting large models with a small number of trainable parameters~\citep{wang2022learning,wang2022dualprompt,hu2022lora}. In particular, many continual learning methods adopt LoRA and allocate a new branch for each task while freezing previous branches to reduce forgetting~\citep{liang2024inflora,wang2023orthogonal,zhao2024sapt}. However, this design often results in a collection of isolated task-specific modules, limiting cross-task knowledge transfer and reducing adaptability. Recent approaches improve forgetting by combining parameter freezing with external gating to select transfer branches~\citep{liang2025gainlora}. Despite these gains, they still generalise poorly to unseen task domains and remain limited in representation fusion across tasks.

\paragraph{Mixture-of-Experts in Continual Learning.}
Recent work has explored Mixture-of-Experts (MoE) architectures to improve stability in continual learning. Prior studies suggest that expert diversification can reduce generalisation error under sequential task learning~\citep{li2025theory}. Existing methods typically isolate task-specific knowledge through input-dependent routing, such as reconstruction-based selectors~\citep{yu2024boosting} or dual-router designs~\citep{huai2025cl}. Other works regularise routing with auxiliary objectives, including consistency distillation~\citep{dai-etal-2022-stablemoe} and orthogonality constraints~\citep{feng2025omoe}, to improve stability and reduce representation collapse. However, these approaches remain vulnerable to semantic shift as new tasks alter the feature space and blur expert boundaries. More importantly, they lack a mechanism to assess representation compatibility before permanent consolidation, making expert updates more vulnerable to parameter interference and forgetting; CP-MoE addresses this gap directly.

\section{Method}

\subsection{Preliminaries}
\label{sec:prelim}

% We first describe the general continual learning setup, the Mixture-of-Experts (MoE) formulation, and the Synaptic Intelligence (SI) importance estimation that underpins our approach.

% \subsubsection{Continual Learning Setting for VQA}
% We consider a sequence of VQA tasks $\{1, 2, \dots, M\}$, where each task $t$ is associated with a dataset $D_t$ containing image-question-answer triplets:
% \[
% D_t = \{(X^{img}_{i}, X^{ques}_{i}, X^{ans}_{i})\}_{i=1}^{N_t}.
% \]
% The learner sequentially trains on task $t$ with the goal of minimizing the prediction error on the current task while mitigating catastrophic forgetting on previous tasks. The optimization objective is to minimize the negative log-likelihood:
% \begin{equation}
% \label{eq:vqa_task_loss}
% \mathcal{L}_{\text{task}}^t(\Theta) = -\mathbb{E}_{(X^{img}, X^{ques}, X^{ans}) \sim D_t} \Big[ \sum_{j=1}^{|X^{ans}|} \log p_\Theta(X^{ans}_j \mid X^{img}, X^{ques}, X^{ans}_{<j}) \Big]
% \end{equation}

\paragraph{Problem Setup}
We consider continual adaptation of a pre-trained foundation model to a sequence of tasks in both unimodal and multimodal settings. Let $\mathcal{T}=\{1,\dots,M\}$ denote a sequence of tasks arriving sequentially, where each task $t \in \mathcal{T}$ is associated with a dataset
\[
D_t=\{(X_i^{(t)}, Y_i^{(t)})\}_{i=1}^{N_t}.
\]
Here, $X_i^{(t)}$ denotes the task input and $Y_i^{(t)}$ the target output sequence. The input may be unimodal, e.g., a text prompt, or multimodal, e.g., an image-text pair.

Let $f_{\Theta_{\mathrm{frozen}},\Phi}$ denote the pre-trained foundation model, where $\Theta_{\mathrm{frozen}}$ are frozen backbone parameters and $\Phi$ are trainable adaptation parameters. In continual learning, the model is fine-tuned sequentially: after task $t-1$, it is adapted to $D_t$ without access to past datasets $\{D_1,\dots,D_{t-1}\}$. The objective is to learn the current task while preserving performance on previous ones. For autoregressive generation, the task loss is
\begin{equation}
\label{eq:task_loss}
\mathcal{L}_{\mathrm{task}}^t(\Phi)
=
-\mathbb{E}_{(X,Y)\sim D_t}
\left[
\sum_{j=1}^{|Y|}
\log p_{\Theta_{\mathrm{frozen}},\Phi}(Y_j \mid X, Y_{<j})
\right].
\end{equation}
We focus on parameter-efficient continual fine-tuning of large foundation models, where full-parameter updates are costly and prone to cross-task interference. In particular, we study \emph{LoRA-MoE}, which combines low-rank adaptation with Mixture-of-Experts (MoE) for efficient and input-dependent adaptation.

\paragraph{Parameter-Efficient MoE (LoRA-MoE)}
Under this setting, the trainable parameters $\Phi$ are introduced as LoRA-MoE modules on top of a frozen backbone. For a frozen feed-forward network (FFN) block $F(\cdot;\Theta_{\mathrm{FFN}})$ and an input hidden state $x \in \mathbb{R}^{d}$, LoRA-MoE keeps $\Theta_{\mathrm{FFN}}$ fixed and augments the block with $n$ parallel low-rank experts. The $i$-th expert is defined as
\begin{equation}
\label{eq:lora_expert}
E_i(x)=B_iA_ix,
\end{equation}
where $A_i \in \mathbb{R}^{r\times d}$ and $B_i \in \mathbb{R}^{k\times r}$ are trainable matrices with rank $r \ll \min(d,k)$. Each expert therefore defines a low-rank adaptation to the frozen FFN, so continual updates remain confined to the lightweight parameter set $\Phi$.

\paragraph{Sparse Gating and Routing}
The contribution of each expert is determined by a routing function
\begin{equation}
\label{eq:gating_function}
G(x)=\mathrm{Softmax}(xW_{\mathrm{gate}}),
\end{equation}
where $G(x)_i$ denotes the routing weight of expert $i$. Let $\mathcal{K}(x)$ denote the indices of the top-$m$ experts selected by $G(x)$. The adapted FFN output is
\begin{equation}
\label{eq:lora_moe_output}
\tilde{F}(x)
=
F(x;\Theta_{\mathrm{FFN}})
+
\frac{{\rho}}{r}\sum_{i\in\mathcal{K}(x)} G(x)_i\,E_i(x),
\end{equation}

{where $\rho > 0$ is a fixed scaling factor and the factor $\rho/r$ makes the magnitude of the aggregated update invariant to the rank $r$. Every adapted FFN block computes $\tilde{F}(x)$ in place of $F(x;\Theta_{\mathrm{FFN}})$, while all other components of the backbone remain unchanged.} This sparse activation enables input-dependent adaptation while decoupling total expert capacity from per-token computation.

Despite these advantages, existing MoE-based adaptation strategies remain poorly suited to continual learning. In particular, they lack mechanisms to integrate new task knowledge into stable experts without harmful interference. As a result, they often face a trade-off: either experts are isolated too strongly, which limits knowledge transfer across related tasks~\citep{liang2025gainlora}, or expert knowledge is merged too aggressively, which overwrites important parameters and causes substantial forgetting~\citep{huai2025cl}.

% \subsubsection{Sparse Gating and Routing}
% \label{sec:routing_prelim}

% The contribution of each expert is determined by a gating function $G(\mathbf{x}) = \text{Softmax}(\mathbf{x}\mathbf{W}_{gate})$. 
% In LoRA-MoE,the output $f(\mathbf{x})$ is the combination of the frozen backbone response and the gated low-rank adaptations:
% \begin{equation}
% \label{eq:lora_moe_output}
%     f(\mathbf{x}) = \mathbf{W}\mathbf{x} + \frac{\alpha}{r} \sum_{i \in \mathcal{K}} G(\mathbf{x})_i E_i(\mathbf{x})
% \end{equation}
% where $\mathcal{K}$ denotes the set of indices for the $k$ experts with the highest gating scores, and $\alpha$ is a constant scaling factor. This sparse activation mechanism within the low-rank Decoupled Parameter Space effectively decouples the model's total capacity from the per-token computational cost \citep{liu2023pushing, huai2025clmoe}.

% ==================================================================
% PART 2: THE CP-MoE FRAMEWORK
% ==================================================================
\subsection{CP-MoE: Consistency-Preserving Mixture-of-Experts}
\label{sec:cp_moe}

To address these limitations, we propose \textbf{CP-MoE} (Consistency-Preserving Mixture-of-Experts), a continual learning framework that integrates new task knowledge into stable experts while mitigating harmful interference and routing instability. Rather than directly updating experts from current-task signals, CP-MoE follows an \emph{assess-then-update} paradigm: it first identifies which experts are suitable for the new task and which parameters should be preserved, and only then performs constrained updates. In this way, CP-MoE avoids the two common failure modes of existing MoE-based continual learning methods: excessive expert isolation, which limits transfer, and overly aggressive merging, which causes forgetting. CP-MoE consists of two key components:
\begin{enumerate}
    \item \textbf{Transient Expert-Guided Parameter Protection.} 
    Before integrating new knowledge into stable experts, CP-MoE uses \emph{transient experts} to capture task-specific adaptation signals. These signals are then used to estimate parameter importance and impose selective regularisation during updating: parameters important for previously acquired knowledge are strongly protected, while less critical parameters remain free to adapt. This enables knowledge transfer without indiscriminate overwriting.

    \item \textbf{Expert Representation Consistency Routing.}
    To stabilise routing during continual adaptation, we introduce a routing objective that encourages consistency between the current input and expert representations. This steers inputs towards experts with compatible representations, while avoiding the rigid uniformity imposed by standard routing objectives and reducing the risk of expert collapse~\citep{chi2022on}.
\end{enumerate}

Together, these two components couple routing and updating: CP-MoE routes inputs to compatible experts and updates them under explicit protection, enabling stable knowledge integration during continual adaptation.

\subsection{Transient Expert}
\label{sec:te}

At the beginning of each task $t$, instead of updating the stable experts $\Phi$ immediately, we instantiate a transient expert $
\phi_t^{\mathrm{TE}}=\{A_t^{\mathrm{TE}}, B_t^{\mathrm{TE}}\},$
implemented as a task-specific LoRA adapter. The transient expert is trained on a warm-up subset $\widehat D_t \subseteq D_t$, while keeping both the frozen backbone $\Theta_{\mathrm{frozen}}$ and the stable LoRA-MoE parameters $\Phi$ fixed. Let $
\widehat{\mathcal L}_{\mathrm{task}}^t(\phi_t^{\mathrm{TE}})$ denote the task loss in Eq.~\ref{eq:task_loss}, evaluated on $\widehat D_t$ and viewed as a function of $\phi_t^{\mathrm{TE}}$ with $(\Theta_{\mathrm{frozen}},\Phi)$ fixed. The transient expert is updated as
\begin{equation}
\label{eq:te_update}
\phi_{t,s+1}^{\mathrm{TE}}
=
\phi_{t,s}^{\mathrm{TE}}
-
\eta \nabla_{\phi}\widehat{\mathcal L}_{\mathrm{task}}^t(\phi_{t,s}^{\mathrm{TE}}),
\qquad s=0,\dots,S-1,
\end{equation}
with {\(S\) the number of warm-up optimisation steps}, and initialisation chosen such that the transient expert has zero initial effect, i.e., \(E_t^{\mathrm{TE}}(x)=0\) at the start of optimisation, where \(\eta\) is the learning rate of the transient expert. Crucially, the transient expert is discarded after warm-up. Instead, we use its optimisation trajectory
\[
\{\phi_{t,s}^{\mathrm{TE}}\}_{s=0}^{S}
\]
to estimate \emph{prospective importance weights} via the path-integral rule~\citep{zenke2017continual}. Let
\[
\Delta \phi_{t,s}^{\mathrm{TE}}
=
\phi_{t,s+1}^{\mathrm{TE}}-\phi_{t,s}^{\mathrm{TE}},
\qquad
g_{t,s}
=
\nabla_{\phi}\widehat{\mathcal L}_{\mathrm{task}}^t(\phi_{t,s}^{\mathrm{TE}}).
\]
{Here $g_{t,s}$ measures the loss sensitivity of each parameter, while $\Delta\phi_{t,s}^{\mathrm{TE}}$ is the movement the optimiser actually made; importance follows from their product rather than gradient magnitude alone.}

{For each parameter index $k$, we accumulate its contribution to the loss reduction achieved over the warm-up trajectory}:
\begin{equation}
\label{eq:omega_te}
\omega_{t,k}
=
\sum_{s=0}^{S-1}
-\, g_{t,s,k}\,\Delta \phi_{t,s,k}^{\mathrm{TE}}.
\end{equation}

{A large $\omega_{t,k}$ indicates that parameter $k$ moved in a direction that consistently reduced the warm-up loss, and is therefore prospectively important for task $t$, whereas the negative sign makes $\omega_{t,k}$ positive for steps that decrease the loss.} After warm-up, the prospective importance is normalised as
\begin{equation}
\label{eq:Omega_te}
\Omega_{t,k}
=
\frac{\omega_{t,k}}
{\bigl(\phi_{t,S,k}^{\mathrm{TE}}-\phi_{t,0,k}^{\mathrm{TE}}\bigr)^2+\xi},
\end{equation}
where \(\xi>0\) is a damping constant. Collecting all entries yields a task-specific importance mask
\[
\Omega_t=\{\Omega_{t,A},\Omega_{t,B}\},
\]
corresponding to the two LoRA factors in
\(
\phi_t^{\mathrm{TE}}=\{A_t^{\mathrm{TE}},B_t^{\mathrm{TE}}\}.
\)
Since the transient and stable experts share the same low-rank parameterisation, \(\Omega_t\) can be transferred to the corresponding stable parameters for protected updating. Hence, the transient expert serves as a disposable look-ahead probe, revealing task-specific adaptation directions before any update is committed to the stable experts. {This probing procedure is illustrated in Figure~\ref{fig:mainfig} (left).}

\begin{theorem}[Transient expert as a local look-ahead estimator]
\label{thm:te_lookahead}
Let \(\widehat{\mathcal L}_{\mathrm{task}}^t(\phi)\) denote the task loss in Eq.~\ref{eq:task_loss}, evaluated on the warm-up subset \(\widehat D_t\), with \((\Theta_{\mathrm{frozen}},\Phi)\) fixed. Suppose that, in a neighbourhood containing the transient trajectory \(\{\phi_{t,s}^{\mathrm{TE}}\}_{s=0}^S\), the warm-up objective admits the quadratic form
\[
\widehat{\mathcal L}_{\mathrm{task}}^t(\phi)
=
\widehat{\mathcal L}_{\mathrm{task}}^t(\phi_{t,0}^{\mathrm{TE}})
+
g_t^\top(\phi-\phi_{t,0}^{\mathrm{TE}})
+
\frac{1}{2}
(\phi-\phi_{t,0}^{\mathrm{TE}})^\top
H_t
(\phi-\phi_{t,0}^{\mathrm{TE}}),
\]
where
\[
g_t
=
\nabla_{\phi}\widehat{\mathcal L}_{\mathrm{task}}^t(\phi_{t,0}^{\mathrm{TE}}),
\qquad
H_t
=
\nabla_{\phi}^2\widehat{\mathcal L}_{\mathrm{task}}^t(\phi_{t,0}^{\mathrm{TE}})
\succeq 0.
\]
If the transient expert is updated by Eq.~\ref{eq:te_update} for \(S\) steps with step size \(\eta < 2/\|H_t\|_2\), then
\begin{equation}
\label{eq:te_filtered_update_subset}
\phi_{t,S}^{\mathrm{TE}}-\phi_{t,0}^{\mathrm{TE}}
=
-\eta \sum_{s=0}^{S-1}(I-\eta H_t)^s g_t
=
- q_S(H_t)\, g_t,
\end{equation}
where $
q_S(\lambda)
=
\eta\sum_{s=0}^{S-1}(1-\eta\lambda)^s.$
Therefore, the transient expert induces a spectrally filtered multi-step adaptation direction for the warm-up objective on \(\widehat D_t\).
\end{theorem}

Theorem~\ref{thm:te_lookahead} shows that the transient expert captures more than an instantaneous gradient on the new task. Although trained only on \(\widehat D_t\), its short warm-up trajectory induces a filtered multi-step adaptation direction for the local task objective, making it a suitable probe for estimating prospective parameter importance before updating the stable experts. This differs from MAML-style look-ahead~\citep{finn2017model} in both role and optimisation. In MAML, the inner-loop updates are differentiated through to learn a meta-initialisation, which typically restricts the number of steps for efficiency. By contrast, the transient expert is not meta-optimised: its warm-up trajectory serves only as a task-local probe of the parameter displacement induced by the current task.

\paragraph{Why transient experts are preferable to direct look-ahead in MoE.}
CP-MoE performs task-specific look-ahead using a single transient expert, while keeping the stable experts fixed during assessment. This makes the procedure lightweight and avoids multi-step optimisation over the stable MoE system. By contrast, direct look-ahead on stable experts would require unrolling updates through routed experts, and potentially the router, making optimisation heavier and more susceptible to cross-expert interference. The transient expert therefore provides a cheap and isolated probe for estimating task-specific importance before protected updates are applied to the stable experts.

{\paragraph{Relation to SI/EWC.}
While Eq.~\ref{eq:omega_te} adopts the path-integral rule of SI~\citep{zenke2017continual}, 
the transient expert changes \emph{where} and \emph{when} importance is estimated. 
SI and EWC estimate importance \emph{retrospectively} on the trained model after each 
task, yielding a single global importance measure. In contrast, the transient expert
instead provides a \emph{prospective} estimate along an \emph{isolated} trajectory: it is trained
briefly with the backbone and all stable experts frozen, so the measured displacement reflects the new task alone, unaffected by concurrent updates to the stable pool or by shifts in routing. The probe is discarded afterwards, and its importance mask is transferred \emph{per expert},
weighted by the CKA-based consistency score (Eq.~\ref{eq:omega_accum}), rather than applied uniformly as in EWC/SI.}

% \subsubsection{Path Integral Importance Estimation}
% We employ a path integral method \citep{zenke2017continual} to estimate the per-parameter importance for learned tasks. This measure is computed by accumulating the gradient contribution over the entire optimization trajectory. For each parameter $\theta_{i},$ we accumulate a contribution score $\omega_{i}$ at each update step $k$:
% \begin{equation}
%     \omega_{i} \leftarrow \omega_{i} - g_{i}(\theta_{k}) \cdot (\theta_{k+1,i} - \theta_{k,i}).
% \end{equation}
% At the end of training for task $t$, the final importance weights $\Omega_{t,i}$ is normalized as:
% \begin{equation}
%     \Omega_{t,i} = \frac{\omega_{i}}{(\theta_{t,i} - \theta_{init,i})^2 + \xi}
% \end{equation}
% This results in a per-parameter local importance mask $\Omega_{t}$.

\subsection{Expert Representation Consistency Routing}

\subsubsection{Motivation: The Tension Between Load Balancing and Representation Consistency}

In modern Mixture-of-Experts (MoE) architectures, load balancing is essential for efficient expert utilisation and hardware efficiency. Earlier methods typically impose this through an \emph{auxiliary load-balancing loss} (\(\mathcal{L}_{\mathrm{aux}}\))~\citep{lepikhin2020gshard,fedus2022switch}, while more recent architectures use \emph{dynamic bias updates} to reduce the optimisation interference of an explicit auxiliary objective~\citep{wang2024auxiliary}.

The issue in \emph{continual learning} is therefore not whether load balancing is needed, but how it is enforced. Because continual data are sequential and highly non-stationary across tasks, overly rigid pressure towards uniform expert usage can conflict with expert specialisation. The router may then assign semantically mismatched inputs to historical experts simply to maintain balanced utilisation, exposing specialised experts to irrelevant updates and causing semantic interference.

To resolve this, we propose \textbf{CP-MoE}, which preserves the benefits of load balancing while making routing \emph{representation-consistent}. Concretely, CP-MoE introduces a structural inductive bias that favours experts whose representations are more compatible with the incoming task, thereby enabling efficient expert usage without allowing load-balancing pressure to override expert specialisation.

\begin{figure}
    \centering
    \includegraphics[width=1\linewidth]{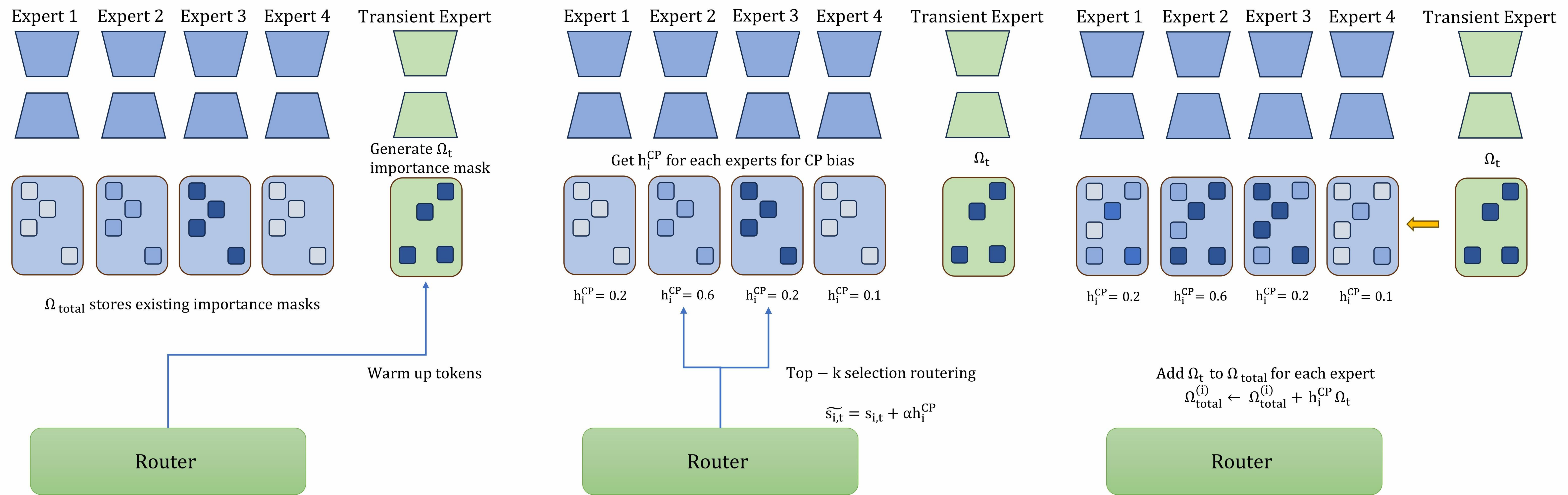}
    \caption{\textbf{Overview of the CP-MoE Framework.} 
    (Left) \textbf{Transient Expert Probing}: A task-specific transient expert (TE) is optimised on warm-up tokens to derive the prospective importance mask $\Omega_t$. 
    (Middle) \textbf{Expert Representation Consistency Routing}: The Centered Kernel Alignment (CKA) between the TE and each stable expert (SE) is measured to produce the representation-consistency scores $h_i^{\mathrm{CP}}$. These scores are subsequently injected as a structural inductive bias to guide the router towards semantically compatible experts, ensuring that load-balancing pressure does not override expert specialisation. 
    (Right) \textbf{Transient Expert-Guided Parameter Protection}: After training on the new task, the importance mask $\Omega_t$ is selectively accumulated into the expert-specific importance matrices $\Omega_{\mathrm{total}}^{(i)}$. This accumulation is weighted by the representation-consistency scores $h_i^{\mathrm{CP}}$, ensuring that stable experts with higher semantic alignment to the current task receive prioritised parameter protection to mitigate catastrophic forgetting.}
    \label{fig:mainfig}
\end{figure}

\subsubsection{Consistency-Preserving Routing Bias (CP Bias)}
\label{sec:cp_bias}

We implement representation-consistent routing via a \emph{consistency-preserving} (CP) bias. Specifically, we use the transient expert (TE) from Sec.~\ref{sec:te} as a task-specific probe, and measure its similarity to each stable expert (SE) using \emph{Centered Kernel Alignment} (CKA).

{Let \(Z^{\mathrm{TE}} \in \mathbb{R}^{N_{\text{warm-up}} \times D}\) and \(Z_i^{\mathrm{SE}} \in \mathbb{R}^{N_{\text{warm-up}} \times D}\) denote the centred activation matrices of the transient expert and the \(i\)-th stable expert, respectively, computed on the same warm-up tokens, where \(N_{\text{warm-up}}\) is the number of warm-up tokens and \(D\) is the hidden dimension}. We define the representation-consistency score of expert \(i\) as
\begin{equation}
\label{eq:cka_score}
h_i^{\mathrm{CP}}
=
\mathrm{CKA}(Z^{\mathrm{TE}}, Z_i^{\mathrm{SE}})
=
\frac{\|(Z_i^{\mathrm{SE}})^\top Z^{\mathrm{TE}}\|_F^2}
{\|(Z^{\mathrm{TE}})^\top Z^{\mathrm{TE}}\|_F \, \|(Z_i^{\mathrm{SE}})^\top Z_i^{\mathrm{SE}}\|_F},
\end{equation}
where \(\|\cdot\|_F\) denotes the Frobenius norm. This gives a task-dependent prior \(h_i^{\mathrm{CP}} \in [0,1]\) for each stable expert.

We inject this prior into routing by biasing the router logits. For token \(t\), let $ s_{i,t}=x_t^\top W_i$
denote the original routing logit for expert \(i\), where \(x_t\) is the token representation and \(W_i\) is the corresponding router parameter. We then define the biased logit
\begin{equation}
\label{eq:cp_biased_logit}
\tilde{s}_{i,t}=s_{i,t}+\alpha h_i^{\mathrm{CP}},
\end{equation}
where \(\alpha \ge 0\) controls the strength of the CP bias. Let \(\mathcal{K}_t\) be the set of top-\(K\) experts under \(\{\tilde{s}_{i,t}\}_{i=1}^n\). The final routing weights are
\begin{equation}
\label{eq:cp_routing_weight}
{G_{i,t}}
=
\begin{cases}
\dfrac{\exp(\tilde{s}_{i,t})}{\sum_{j \in \mathcal{K}_t}\exp(\tilde{s}_{j,t})}, & i \in \mathcal{K}_t, \\[6pt]
0, & \text{otherwise}.
\end{cases}
\end{equation}

{During training, the routing weights $G_{i,t}$ defined in Eq.~\ref{eq:cp_routing_weight} replace $G(x)_i$ in the block output of Eq.~\ref{eq:lora_moe_output}. They therefore determine how expert contributions are combined at every adapted FFN block, and enter $\mathcal{L}_{\mathrm{task}}$ (Eq.~\ref{eq:task_loss}) through the resulting token logits.} The CP bias acts as a stabilising prior: it favours experts whose representations are more compatible with the incoming task, while preserving sparse routing and standard load balancing. Consequently, expert utilisation remains efficient, but routing is less likely to send semantically mismatched inputs to historical experts, reducing routing instability and semantic interference during continual adaptation. {This routing mechanism is illustrated in Figure~\ref{fig:mainfig} (middle).}

\paragraph{Load-Balancing Auxiliary Loss.}
While the CP bias stabilises routing by favouring representation-consistent experts, a load-balancing constraint is still required to prevent expert collapse. We therefore retain the standard auxiliary load-balancing loss
\begin{equation}
\label{eq:aux_loss}
\mathcal{L}^t_{\mathrm{aux}}=\sum_{i=1}^{n} f_i P_i,
\end{equation}
where \(f_i\) is the fraction of tokens routed to expert \(i\) in the current batch, and $P_i=\frac{1}{T}\sum_{t=1}^{T}\mathrm{Softmax}(s_{i,t})$
is the mean native routing probability of expert \(i\) over the \(T\) tokens in the batch.

Importantly, \(P_i\) is computed only from the native routing logits \(s_{i,t}\), without the CP bias. This keeps load balancing decoupled from the consistency prior: the CP bias guides routing towards representation-compatible experts, while \(\mathcal{L}_{\mathrm{aux}}\) maintains balanced expert usage.

\subsection{Transient Expert-Guided Parameter Protection}
\label{sec:protection}

To mitigate catastrophic forgetting, CP-MoE penalises drift in stable expert parameters using accumulated importance weights. Rather than applying a binary mask, we maintain a continuous importance matrix that selectively regularises parameters according to their estimated relevance to previously acquired knowledge.

\paragraph{Importance accumulation.}
Let \(\Omega_t=\{\Omega_{t,A},\Omega_{t,B}\}\) denote the task-specific prospective importance mask estimated from the transient expert. For each stable expert \(E_i\), we accumulate importance as
\begin{equation}
\label{eq:omega_accum}
\Omega^{(i)}_{A,\mathrm{total}}
\leftarrow
\Omega^{(i)}_{A,\mathrm{total}} + h_i^{\mathrm{CP}} \,\Omega_{t,A},
\qquad
\Omega^{(i)}_{B,\mathrm{total}}
\leftarrow
\Omega^{(i)}_{B,\mathrm{total}} + h_i^{\mathrm{CP}} \,\Omega_{t,B},
\end{equation}
where \(h_i^{\mathrm{CP}}\) is the CKA-based consistency score in Eq.~\ref{eq:cka_score}. Thus, experts that are more aligned with the incoming task receive stronger protection, while unrelated experts are not over-constrained.

\paragraph{Regularisation on low-rank adapters.}
We apply the accumulated importance directly to the LoRA parameters of each stable expert. Let \(A_i^{\mathrm{old}}\) and \(B_i^{\mathrm{old}}\) denote the parameter snapshots saved after the previous task. The regularisation term is
\begin{equation}
\label{eq:reg_loss_final}
\mathcal{L}_{\mathrm{reg}}
=
\sum_{i=1}^{n}
\left(
\left\langle
\Omega^{(i)}_{A,\mathrm{total}},
(A_i-A_i^{\mathrm{old}})^{\odot 2}
\right\rangle
+
\left\langle
\Omega^{(i)}_{B,\mathrm{total}},
(B_i-B_i^{\mathrm{old}})^{\odot 2}
\right\rangle
\right),
\end{equation}
where \(\langle U,V\rangle\) denotes the Frobenius inner product. This regulariser preserves historically important directions within the low-rank adaptation subspace. {This protection mechanism is illustrated in Figure~\ref{fig:mainfig} (right).}

\paragraph{Overall objective.}
For task \(t\), the final training objective combines the task loss, the importance-weighted regularisation, and the auxiliary load-balancing loss:
\begin{equation}
\label{eq:total_loss}
\mathcal{L}_{\mathrm{total}}^{t}
=
\mathcal{L}_{\mathrm{task}}^{t}
+
\lambda\,\mathcal{L}_{\mathrm{reg}}^{t}
+
\gamma\,\mathcal{L}_{\mathrm{aux}}^{t},
\end{equation}
where \(\lambda\) and \(\gamma\) are balancing coefficients.

\section{Experiments}
\label{section:experiments}

\paragraph{Datasets.}
We evaluate CP-MoE on continual learning benchmarks in both unimodal language and multimodal vision--language settings.

\begin{itemize}
    \item \textbf{Super-NaturalInstructions (SuperNI).} To assess continual adaptation in long-sequence language tasks, we use SuperNI~\citep{wang2022supernaturalinstructions}, which covers diverse NLP tasks including summarization, information extraction, and dialogue generation. Each task is cast as an instruction-following text generation problem.
    
    \item \textbf{VQA v2.} For multimodal evaluation, we use VQA v2~\citep{goyal2017making}, which contains over 200k images and 1.1M questions. Following the VQACL protocol~\citep{zhang2023vqacl}, we split the benchmark into 10 sequential tasks by reasoning type: recognition, location, judge, commonsense, count, action, color, type, subcategory, and causal. This setting tests whether CP-MoE can mitigate catastrophic forgetting in multimodal representation learning.
\end{itemize}

\paragraph{Evaluation metrics.} We report three metrics to capture the plasticity--stability trade-off:
\begin{itemize}
    \item \textbf{Average Performance (AP):} the average performance over all seen tasks after the final training stage, measured by accuracy for VQA v2 and ROUGE-L for SuperNI.
    \item \textbf{Average Forgetting (AF/AF\textsuperscript{*}):} {AF is the signed average drop from each task's post-training performance to its final value (negative indicates beneficial backward transfer); AF\textsuperscript{*} is the conventional non-negative variant, averaging each task's drop from its best to its final performance.}

    \item \textbf{Zero-Shot Transfer (ZST):} for SuperNI, the average performance on 7 unseen tasks after completing the 8-task training sequence, measuring the model's out-of-distribution generalisation after continual learning.
\end{itemize}

\paragraph{Implementation details.} For unimodal language experiments, we use Llama-2-7B~\citep{touvron2023llama} as the backbone, while for multimodal experiments we use LLaVA-1.5-7B~\citep{liu2024improved}. In all settings, CP-MoE is instantiated with LoRA-based experts while the backbone remains frozen, and all methods are optimised with AdamW using a learning rate of \(2\times 10^{-4}\) and a cosine decay schedule.

For unimodal experiments, each task is trained for 5 epochs with a per-device batch size of 16. The transient expert is warm-started for 10{,}000 tokens. We insert LoRA experts into all dense layers, including the self-attention projections \texttt{(q\_proj, k\_proj, v\_proj, o\_proj)} and MLP layers \texttt{(gate\_proj, up\_proj, down\_proj)}. We use 8 experts with rank \(r=4\) and LoRA dropout 0.1. The CP-bias coefficient is set to \(\alpha=0.2\), while the balancing coefficients are \(\lambda=5\times10^3\) and \(\gamma=0.1\).

For VQAv2, we keep the same base optimisation hyperparameters but adapt the architecture to the multimodal setting. Specifically, training is limited to 1 epoch, LoRA experts are injected only into the MLP layers, and we use a more compact configuration with 4 experts of rank \(r=4\). In addition, the transient-expert warm-up is increased to 100{,}000 tokens to provide a more stable initialisation for multimodal routing.

% --- TABLE 1: SUPERNI SUMMARY (MAIN RESULT) ---
\begin{table}[htbp]

\centering
\small 
\setlength{\tabcolsep}{5.5pt}
\renewcommand{\arraystretch}{1.15}
\caption{Performance comparison on SuperNI. AP: Average Performance; AF: Average Forgetting
(signed, based on backward transfer); AF*: conventional Forgetting, the average drop from each
task's best to its final performance; ZST: Zero-shot Transfer. \textbf{Bold} indicates the best
result.}
\label{tab:main_results}
\begin{tabular}{l cccc cccc}
\toprule
\multirow{2}{*}{Method} & \multicolumn{4}{c}{Order 1} & \multicolumn{4}{c}{Order 2} \\
\cmidrule(lr){2-5} \cmidrule(lr){6-9}
 & AP $\uparrow$ & AF $\downarrow$ & AF$^*$
 $\downarrow$ & ZST $\uparrow$ & AP $\uparrow$ & AF $\downarrow$ & AF$^*$
 $\downarrow$ & ZST $\uparrow$ \\
\midrule
InfLoRA & 47.28 & 1.05 & 1.81 & 27.43 & 50.06 & \textbf{-0.28} & 0.63 & 30.79 \\
O-LoRA  & 50.12 & \textbf{-2.10} & 1.23 & 33.30 & 45.20 & 0.37 & 2.34 & 27.43 \\
GainLoRA (Inf) & 45.57 & -0.28 & \textbf{0.61} & 27.77 & 47.27 & 0.30 & \textbf{0.52} & 30.86 \\
GainLoRA (O) & 49.60 & 0.82 & 1.73 & 33.80 & \textbf{51.54} & -0.21 & 0.55 & 33.64 \\
\midrule
\rowcolor{gray!10} \textbf{CP-MoE (Ours)} & \textbf{50.84} & 0.62 & 1.32 & \textbf{35.80} & 50.56 & 0.73 & 1.18 & \textbf{35.91} \\
\bottomrule
\end{tabular}
\end{table}

\subsection{Main Results}
\label{subsec:results}

% \paragraph{Unimodal Language Generation and Zero-shot Generalization.}
% The results on the SuperNI benchmark (Table~\ref{tab:main_results}, Table~\ref{tab:main_tasks}, and Table~\ref{tab:zeroshot_tasks}) further validate the robustness of CP-MoE. As shown in Table~\ref{tab:main_results}, CP-MoE achieves the highest Average Performance (AP) of 50.84\% and Zero-shot Transfer (ZST) of 35.80 in Order 1.
\paragraph{Unimodal Language Generation and Zero-shot Generalisation.}
Table~\ref{tab:main_results} compares CP-MoE against MoE-based fine-tuning baselines on SuperNI with Llama-2-7B. In Order 1, CP-MoE achieves the best overall results, attaining the highest Average Performance (AP) of 50.84\% and Zero-shot Transfer (ZST) of 35.80\%.

In Order 2, CP-MoE again achieves a strong AP of 50.56\%, remaining highly competitive with the best GainLoRA variant (51.54\%). However, our analysis suggests that the apparent advantage of GainLoRA in this setting is partly attributable to misaligned evaluation metrics and instruction formatting, which lead to inflated scores. We provide a detailed case analysis in Appendix~\ref{subsec:qualitative}.

To further assess out-of-distribution generalisation, we evaluate zero-shot transfer on unseen tasks (Tasks 9--15). As shown in Table~\ref{tab:main_results}, CP-MoE achieves a ZST score of 35.80\%, substantially outperforming GainLoRA-olora (33.80\%) and GainLoRA-infolora (27.77\%). These results suggest that grounding routing decisions in intrinsic consistency, rather than rigid load-balancing constraints, helps preserve a more stable and transferable representation space through continual updates. In addition, our importance-based parameter protection regularises weight updates to better retain pre-trained knowledge while allowing task-relevant adaptation. Together, these properties promote both stronger expert specialisation and more effective knowledge transfer, which in turn lead to better out-of-distribution performance.

% --- TABLE 4: VQA V2 SOTA COMPARISON ---
\begin{table*}[ht]
\centering
\caption{\textbf{Performance (\%) comparison on VQA v2.} We compare CP-MoE against state-of-the-art continual learning methods. The top two sections cite results from \citep{huai2025cl}. The bottom section reports our reproduction under a unified experimental setting to ensure fair comparison. Our CP-MoE achieves the highest Average Performance (AP) and minimal Average Forgetting (AF).}
\label{tab:sota_comparison}
\resizebox{\textwidth}{!}{%
\begin{tabular}{lcccccccccccc}
\toprule
\multirow{2.5}{*}{\textbf{Methods}} & \multicolumn{10}{c}{\textbf{Various task in VQA v2}} & \multirow{2.5}{*}{$AP(\uparrow)$} & \multirow{2.5}{*}{$AF(\downarrow)$} \\ 
\cmidrule(lr){2-11}
 & Rec. & Loc. & Jud. & Com. & Cou. & Act. & Col. & Typ. & Sub. & Cau. & & \\ 
\midrule
\multicolumn{13}{l}{\textit{VL-T5 based methods (Reported in \citep{huai2025cl})}} \\ 
\midrule
Vanilla \citep{cho2021unifying} & 7.39 & 4.94 & 22.29 & 32.30 & 0.71 & 12.14 & 12.10 & 10.69 & 27.29 & 15.10 & 14.49 & 30.15 \\
EWC \citep{kirkpatrick2017overcoming} & 6.73 & 8.43 & 27.22 & 47.10 & 0.14 & 12.40 & 1.76 & 10.98 & 31.05 & 11.85 & 15.77 & 28.38 \\
MAS \citep{aljundi2018memory} & 30.81 & 8.07 & 25.50 & 4.00 & 31.90 & 32.39 & 26.24 & 24.75 & 19.85 & 2.75 & 20.56 & 21.97 \\
ER \citep{chaudhry2019continual} & 18.64 & 21.36 & 61.27 & 64.17 & 30.29 & 52.84 & 43.39 & 23.31 & 42.75 & 11.85 & 36.99 & 4.80 \\
DER \citep{buzzega2020dark} & 14.55 & 13.83 & 62.88 & 65.16 & 30.96 & 51.19 & 40.51 & 19.04 & 42.87 & 12.55 & 35.35 & 6.58 \\
VS \citep{wan2022continual} & 15.66 & 19.21 & 59.86 & 32.16 & 27.28 & 47.79 & 32.32 & 20.44 & 41.38 & 10.20 & 34.03 & 11.68 \\
VQACL \citep{zhang2023vqacl}& 20.47 & 28.02 & 62.55 & 68.61 & 29.35 & 50.66 & 44.45 & 26.36 & 44.65 & 12.60 & 38.77 & 2.90 \\ 
\midrule
\multicolumn{13}{l}{\textit{LLaVA-7B based methods (Reported in \citep{huai2025cl})}} \\ 
\midrule
Vanilla & 19.25 & 14.81 & 54.59 & 56.97 & 24.23 & 46.20 & 27.58 & 26.09 & 36.47 & 18.89 & 32.51 & 20.69 \\
EWC & 28.12 & 23.02 & 61.50 & 61.08 & 26.13 & 54.29 & 23.65 & 32.25 & 44.97 & 17.83 & 37.28 & 15.27 \\
MAS & 31.54 & 22.09 & 60.85 & 46.32 & 32.48 & 56.47 & 30.05 & 35.69 & 42.73 & 18.83 & 37.71 & 14.91 \\
ER & 29.31 & 25.74 & 63.46 & 65.78 & 31.92 & 58.39 & 45.17 & 34.55 & 46.24 & 18.96 & 41.95 & 10.20 \\
DER & 26.95 & 21.43 & 64.88 & 66.17 & 31.01 & 55.92 & 44.60 & 32.85 & 47.09 & 20.74 & 41.16 & 11.28 \\
VS & 28.48 & 24.09 & 61.37 & 67.20 & 29.56 & 54.64 & 33.98 & 32.91 & 45.82 & 19.89 & 39.79 & 12.70 \\
VQACL & 34.14 & 32.19 & 66.15 & 63.00 & 33.01 & 60.91 & 34.64 & 38.48 & 47.94 & 24.42 & 43.49 & 9.10 \\ 
\midrule
\multicolumn{13}{l}{\textit{Unified MoE Implementation (Ours)}} \\ 
\midrule
CL-MoE \citep{huai2025cl} & 55.74 & 41.74 & 80.39 & 76.77 & 49.47 & 75.41 & 73.56 & 63.17 & 61.03 & \textbf{30.41} & 60.77 & 1.77 \\
\textbf{CP-MoE (Ours)} & \textbf{57.96} & \textbf{43.73} & \textbf{82.52} & \textbf{79.05} & \textbf{52.87} & \textbf{77.21} & \textbf{74.79} & \textbf{64.16} & \textbf{62.41} & 29.95 & \textbf{62.30} & \textbf{0.35} \\ 
\bottomrule
\end{tabular}%
}
\end{table*}

\paragraph{VQA v2 Benchmark.}
Table~\ref{tab:sota_comparison} reports the comparative results on the VQA v2 benchmark. CP-MoE consistently outperforms both rehearsal-based methods and recent MoE-based baselines in the multimodal task-incremental setting. In particular, CP-MoE achieves the best \textbf{Average Performance (AP)} of 62.30\%. More importantly, it attains a near-zero \textbf{Average Forgetting (AF)} of \textbf{0.35\%}, substantially improving over the previous state of the art, CL-MoE (1.77\%). These results show that CP-MoE not only delivers strong performance in unimodal language continual learning, but also generalises effectively to multimodal settings while mitigating catastrophic forgetting. This advantage over CL-MoE stems from two key design differences: first, CP-MoE explicitly accounts for representation similarity during routing, whereas CL-MoE ignores it; second, CP-MoE preserves task-specific important parameters, while CL-MoE relies on simple linear weighting that is more prone to overwriting previously acquired knowledge.

\begin{table}[ht]
{
\centering
\small
\caption{\textbf{Ablation Study on SuperNI Main Tasks (1-8), mean$_{\pm\text{std}}$ over 5 seeds.}  CP Bias: Consistency-Preserving Bias, TE Reg.: Transient Expert Regularisation. In the lower block, the row without checkmark removes the CKA-based importance mask.}
\label{tab:ablation_main_components}
\resizebox{\columnwidth}{!}{%
\begin{tabular}{ccccccccccccc}
\toprule
\multicolumn{2}{c}{\textbf{Modules}} & 1572 & 363 & 1290 & 181 & 002 & 1510 & 639 & 1729 & ACC & AF \\
\textbf{CP Bias} & \textbf{TE Reg.} & & & & & & & & & & \\
\midrule
-- & -- & 17.15 & 87.67 & 27.00 & 65.50 & 64.62 & 97.82 & 11.89 & 16.26 & 48.49$_{\pm1.27}$ & 2.28$_{\pm0.49}$ \\
-- & \checkmark & 28.74 & 87.00 & 28.07 & 61.71 & 72.99 & 96.97 & 10.61 & 16.88 & 50.37$_{\pm0.39}$ & \textbf{0.76}$_{\pm0.10}$ \\
\checkmark & \checkmark & {32.60} & 88.00 & 28.39 & 57.44 & {74.43} & 97.94 & 11.59 & {17.22} & \textbf{50.95}$_{\pm0.62}$ & 0.86$_{\pm0.05}$ \\
\midrule
\multicolumn{2}{c}{\textbf{with CKA}} & 1572 & 363 & 1290 & 181 & 002 & 1510 & 639 & 1729 & ACC & AF \\
\midrule
\multicolumn{2}{c}{--} & 33.95 & 88.00 & 26.02 & 54.81 & 72.20 & 91.02 & 8.72 & 14.71 & 48.68$_{\pm1.29}$ & 1.64$_{\pm0.14}$ \\
\multicolumn{2}{c}{\checkmark} & {32.60} & {88.00} & 28.39 & 57.44 & {74.43} & 97.94 & 11.59 & {17.22} & \textbf{50.95}$_{\pm0.62}$ & \textbf{0.86}$_{\pm0.05}$ \\
\bottomrule

\end{tabular}%
}}
\end{table}

\subsection{Ablation Study}

Due to space constraints, we present only the core ablations on the main components of CP-MoE, including the effect of the CKA-guided importance mask, in the main paper. Additional ablation studies are deferred to the Appendix.

\paragraph{Main components of CP-MoE.}
We examine the contribution of the key components of CP-MoE, namely Consistency-Preserving Bias (CP-Bias) and Transient Expert Regularisation (TE-Reg) in Table~\ref{tab:ablation_main_components}. Starting from the baseline (ACC {48.49}\%, AF {2.28}\%), adding TE-Reg substantially reduces Average Forgetting (AF) to {0.76}\% while improving Average Performance (ACC) to {50.37}\%. Incorporating CP-Bias further improves ACC to 50.95\%, at the cost of a marginal increase in AF (0.76 → 0.86), which remains well below the baseline (2.28). These consistent gains indicate that TE-Reg and CP-Bias play complementary roles: TE-Reg protects previously acquired knowledge from being overwritten, while CP-Bias improves routing consistency, together mitigating catastrophic forgetting and improving overall performance.

\paragraph{Effect of the CKA-based consistency score.}
Table~\ref{tab:ablation_main_components} also highlights the contribution of the CKA-based consistency score. When \(h_i^{\mathrm{CP}}\) is removed and the same regularisation strength is applied to all experts, performance deteriorates noticeably: average forgetting rises from {0.86} to {1.64}, while average accuracy drops by about {2}\%. This indicates that a uniform regularisation scheme cannot capture the different roles of individual experts. By contrast, the CKA-based consistency score modulates the regularisation strength according to representation alignment, leading to better preservation of task-specific knowledge and less cross-task interference.

{
\paragraph{Effect of the load-balancing loss.}
We further ablate the auxiliary load-balancing loss by setting $\gamma=0$. 
As shown in Table~\ref{tab:gamma_ablation}, removing $\mathcal{L}_{\mathrm{aux}}$ 
leads to unbalanced expert utilisation and a substantial performance drop 
(AP 45.97 vs.\ 50.84, AF 2.08 vs.\ 0.62), confirming that a small $\gamma$ 
is necessary to prevent expert collapse in our fixed-pool setting, while the 
CP bias remains responsible for steering routing 
(recall that $\mathcal{L}_{\mathrm{aux}}$ is computed on the native logits 
$s_{i,t}$ and is decoupled from the CP bias).
}

\begin{table}[ht]
\centering
\small

\caption{Ablation of the load-balancing coefficient $\gamma$ on SuperNI (Order 1).}
\label{tab:gamma_ablation}
\begin{tabular}{ccc}
\toprule
$\gamma$ & AP ($\uparrow$) & AF ($\downarrow$) \\
\midrule
0 & 45.97 & 2.08 \\
0.1 & \textbf{50.84} & \textbf{0.62} \\
\bottomrule
\end{tabular}
\end{table}

\section{Conclusion}
\label{sec:conclusion}

We presented \emph{CP-MoE}, a continual learning framework for LoRA-based Mixture-of-Experts that mitigates forgetting while preserving effective knowledge transfer in both unimodal and multimodal settings. By introducing a transient expert as a lightweight probe, together with consistency-preserving routing and representation-guided regularisation, CP-MoE enables more selective consolidation and stronger protection of important historical knowledge without dynamic expert expansion or persistent computational overhead. Experiments on SuperNI and VQA v2 showed that CP-MoE achieved state-of-the-art performance, lower forgetting, and stronger zero-shot generalisation than strong MoE and PEFT baselines. We hope this work encourages future research on more adaptive and scalable continual learning methods for MoE-based foundation models.

\section*{Acknowledgements}
This work was supported by the US Army International Technology Center Pacific (ITC-IPAC) under Contract No.\ FA520923C0020, and by the ARC Training Centre for Whole Life Design of Carbon Neutral Infrastructure (IC230100015). The authors acknowledge support from the ARC Centre of Excellence for Automated Decision-Making and Society (CE200100005). Computations were performed on the Katana computational cluster, supported by Research Technology Services at UNSW Sydney.

\newpage

\bibliography{collas2026_conference}
\bibliographystyle{collas2026_conference}

\newpage
\appendix
\section{Theoretical Results}

\subsection{Proof for Theorem 1}

\begin{proof}
Under the quadratic model,
\[
\nabla_\phi \widehat{\mathcal L}_{\mathrm{task}}^t(\phi)
=
g_t + H_t(\phi-\phi_{t,0}^{\mathrm{TE}}).
\]
Define
\[
\delta_s := \phi_{t,s}^{\mathrm{TE}}-\phi_{t,0}^{\mathrm{TE}}.
\]
Then \(\delta_0=0\), and the update in Eq.~\ref{eq:te_update} becomes
\[
\delta_{s+1}
=
\delta_s - \eta\bigl(g_t + H_t\delta_s\bigr)
=
(I-\eta H_t)\delta_s - \eta g_t.
\]
Unrolling this recursion from \(\delta_0=0\) yields
\[
\delta_S
=
-\eta \sum_{s=0}^{S-1}(I-\eta H_t)^s g_t.
\]
Since \(\delta_S=\phi_{t,S}^{\mathrm{TE}}-\phi_{t,0}^{\mathrm{TE}}\), we obtain
\[
\phi_{t,S}^{\mathrm{TE}}-\phi_{t,0}^{\mathrm{TE}}
=
-\eta \sum_{s=0}^{S-1}(I-\eta H_t)^s g_t.
\]

Now define
\[
q_S(\lambda):=\eta\sum_{s=0}^{S-1}(1-\eta\lambda)^s.
\]
By polynomial functional calculus,
\[
q_S(H_t)=\eta\sum_{s=0}^{S-1}(I-\eta H_t)^s,
\]
hence
\[
\phi_{t,S}^{\mathrm{TE}}-\phi_{t,0}^{\mathrm{TE}}
=
- q_S(H_t)\, g_t.
\]

To interpret this operator, let \(H_t = U\Lambda U^\top\), where
\(\Lambda=\mathrm{diag}(\lambda_1,\dots,\lambda_d)\) with \(\lambda_j\ge 0\). Then
\[
q_S(H_t)=U\,q_S(\Lambda)\,U^\top,
\qquad
q_S(\Lambda)=\mathrm{diag}\bigl(q_S(\lambda_1),\dots,q_S(\lambda_d)\bigr).
\]
Thus each eigendirection of \(H_t\) is rescaled by the scalar gain \(q_S(\lambda)\). Moreover, for \(\lambda>0\),
\[
q_S(\lambda)
=
\frac{1-(1-\eta\lambda)^S}{\lambda},
\qquad
q_S(0)=\eta S.
\]
Hence flatter directions (\(\lambda\) small) are accumulated more strongly across the \(S\) warm-up steps, while higher-curvature directions are relatively attenuated. Therefore, the transient expert does not merely follow the one-step gradient \(g_t\), but estimates a local multi-step adaptation direction that favours directions remaining useful under repeated descent on \(\widehat D_t\). This proves the claim.
\end{proof}

\section{Detailed Experimental Setup}
\label{sec:appendix_setup}

\paragraph{Implementation Details}
Following the established protocols in GainLoRA, all methods in our SuperNI experiments are implemented using the instruction tuning paradigm. Models are optimised using the AdamW optimiser. All experiments are conducted on NVIDIA H200 GPUs to ensure consistent computational environments.

\paragraph{SuperNI Benchmark Tasks}
To evaluate the zero-shot generalisation and continual learning capabilities, we select 15 tasks from the SuperNI benchmark. These tasks are categorised into five types, as detailed in Table~\ref{tab:superni_details}. For generation tasks, we report \textbf{ROUGE-L}; for classification tasks, we report \textbf{Accuracy}.

% --- TABLE 1: TASK MAPPING ---
\begin{table}[ht]
\centering
\small
\caption{Details of the 15 selected tasks in the SuperNI Benchmark.}
\label{tab:superni_details}
\resizebox{0.9\columnwidth}{!}{%
\begin{tabular}{clll}
\toprule
\textbf{Task ID} & \textbf{Dataset Name} & \textbf{Task Type} & \textbf{Metric} \\
\midrule
1572 & samsum\_summary & Question Answering & Rouge-L \\
363 & sst2\_polarity\_classification & Sentiment Analysis & Accuracy \\
1290 & xsum\_summarization & Question Answering & Rouge-L \\
181 & outcome\_extraction & Info Extraction & Rouge-L \\
002 & quoref\_answer\_generation & Dialogue Gen & Rouge-L \\
1510 & evaluation\_relation\_extraction & Info Extraction & Rouge-L \\
639 & multi\_woz\_user\_utterance\_generation & Summarization & Rouge-L \\
1729 & personachat\_generate\_next & Summarization & Rouge-L \\
073 & commonsenseqa\_answer\_generation & Dialogue Gen & Rouge-L \\
1590 & diplomacy\_text\_generation & Summarization & Rouge-L \\
748 & glucose\_reverse\_cause\_event\_detection & Info Extraction & Rouge-L \\
511 & reddit\_tifu\_long\_text\_summarization & Question Answering & Rouge-L \\
591 & sciq\_answer\_generation & Dialogue Gen & Rouge-L \\
1687 & sentiment140\_classification & Sentiment Analysis & Accuracy \\
875 & emotion\_classification & Sentiment Analysis & Accuracy \\
\bottomrule
\end{tabular}%
}
\end{table}

\paragraph{Task Sequences (Orders)}
In line with the methodology of GainLoRA, we evaluate CP-MoE under two distinct task sequences (Order 1 and Order 2) to ensure the robustness of our results. The specific sequences of Task IDs are presented in Table~\ref{tab:task_orders}.

% --- TABLE 2: TASK SEQUENCE ---
\begin{table}[ht]
\centering
\small
\caption{Task sequences for the SuperNI benchmark.}
\label{tab:task_orders}
\resizebox{0.9\columnwidth}{!}{%
\begin{tabular}{lcl}
\toprule
\textbf{Benchmark} & \textbf{Order} & \multicolumn{1}{c}{\textbf{Task Sequence}} \\
\midrule
\multirow{6}{*}{SuperNI} & \multirow{3}{*}{1} & task1572 $\rightarrow$ task363 $\rightarrow$ task1290 $\rightarrow$ task181 $\rightarrow$ task002 $\rightarrow$ \\
  &  & task1510 $\rightarrow$ task639 $\rightarrow$ task1729 $\rightarrow$ task073 $\rightarrow$ task1590 $\rightarrow$ \\
  &  & task748 $\rightarrow$ task511 $\rightarrow$ task591 $\rightarrow$ task1687 $\rightarrow$ task875 \\
\cmidrule{2-3}
  & \multirow{3}{*}{2} & task748 $\rightarrow$ task073 $\rightarrow$ task1590 $\rightarrow$ task639 $\rightarrow$ task1572 $\rightarrow$ \\
  &  & task1687 $\rightarrow$ task591 $\rightarrow$ task363 $\rightarrow$ task1510 $\rightarrow$ task1729 $\rightarrow$ \\
  &  & task181 $\rightarrow$ task511 $\rightarrow$ task002 $\rightarrow$ task1290 $\rightarrow$ task875 \\
\bottomrule
\end{tabular}%
}
\end{table}

\subsection{Hyperparameter Sensitivity Analysis}
\label{subsec:sensitivity}

We conduct a sensitivity analysis on the routing bias coefficient ($\alpha$) and the rigidity constraint weight ($\lambda$). As shown in Table~\ref{tab:hp_sensitivity}, CP-MoE maintains consistently high Average Performance (AP) and low Average Forgetting (AF) across a broad range of configurations. The model experiences severe forgetting only when $\alpha \to 0$ (effectively removing the CKA-guided semantic consistency), validating that our structural inductive bias is the primary driver of performance rather than meticulous hyperparameter tuning.

% --- TABLE 3: SENSITIVITY ---
\begin{table}[t]

\centering
\caption{Hyperparameter Sensitivity Analysis. Performance on SuperNI when varying the routing bias ($\alpha$), rigidity constraint ($\lambda$), and transient-expert warm-up length. The optimal configuration is highlighted in bold.}
\label{tab:hp_sensitivity}
\begin{tabular}{ccccccc}
\toprule
Bias ($\alpha$) & Rigidity ($\lambda$) & Warm-up & AP ($\uparrow$) & AF ($\downarrow$) & Zero-shot ($\uparrow$) \\
\midrule
0.1 & $5 \times 10^3$ & 10k & 50.48 & 0.67 & \textbf{36.17} \\
\textbf{0.2} & $\mathbf{5 \times 10^3}$ & \textbf{10k} & \textbf{50.84} & \textbf{0.62} & 35.80 \\
0.3 & $5 \times 10^3$ & 10k & 50.08 & 1.46 & 36.09 \\
0.4 & $5 \times 10^3$ & 10k & 50.38 & 0.88 & 35.14 \\
\midrule
0.2 & $1 \times 10^3$ & 10k & \textbf{51.07} & \textbf{0.43} & 35.62 \\
0.2 & $5 \times 10^3$ & 10k & 50.84 & 0.62 & \textbf{35.80} \\
0.2 & $1 \times 10^4$ & 10k & 50.15 & 1.43 & 35.39 \\
\midrule
0.2 & $5 \times 10^3$ & \textbf{10k} & 50.84 & \textbf{0.62} & 35.80 \\
0.2 & $5 \times 10^3$ & 20k & 50.75 & 2.52 & 36.07 \\
0.2 & $5 \times 10^3$ & 30k & \textbf{51.13} & 2.75 & \textbf{36.87} \\
\bottomrule
\end{tabular}
\end{table}

\subsection{Evaluation metrics and instruction analysis}
\label{subsec:qualitative}

While CP-MoE exhibits lower automated scores (e.g., ROUGE, Exact Match) on specific QA datasets compared to the aggressively fine-tuned LoRA branch, a manual review reveals that these static metrics systematically misrepresent the actual knowledge retention of the models. The performance gap predominantly stems from formatting misalignment and metric rigidity rather than catastrophic forgetting.

\paragraph{Task Instruction Setup}
During evaluation, models are prompted with the following instruction:

\begin{quote}
\textit{Definition: Given a scientific question, generate a correct answer to it.} \\
\textit{Now complete the following example -} \\
\textit{Input: \{Question\}} \\
\textit{Output:}
\end{quote}

This instruction requires the model to generate a correct answer but does \textbf{not} impose any hard constraints on the output length (e.g., restricting the answer to 1-2 words). This open-ended formulation allows for both concise entities and complete explanatory sentences. As illustrated in the case studies below, $n$-gram matching metrics fail to capture semantic equivalence and actively penalise grammatically complete answers.

\paragraph{Case 1: False Positives via Lexical Overlap}
\begin{itemize}
    \item \textbf{Question:} Particulates cause lung diseases. They can also increase the risk of heart disease and the number of what?
    \item \textbf{Ground Truth:} asthma attacks
    \item \textbf{LoRA Branch:} heart attacks
    \item \textbf{CP-MoE:} strokes
\end{itemize}

\textbf{Analysis:} Medical consensus links particulate pollution to an increased risk of strokes. CP-MoE outputs a factually correct prediction. The LoRA branch repeats the word ``heart'' from the prompt, producing a redundant prediction. However, because the LoRA prediction shares the token ``attacks'' with the ground truth, it receives an artificially high ROUGE score, whereas CP-MoE scores zero. The metric registers a false positive based on spurious lexical overlap while penalising factual accuracy.

\paragraph{Case 2: Penalising Semantic Variation and Syntactic Completeness}
\begin{itemize}
    \item \textbf{Question:} Why do birds need a light-weight body?
    \item \textbf{Ground Truth:} to stay aloft
    \item \textbf{LoRA Branch:} to fly
    \item \textbf{CP-MoE:} Birds need light-weight bodies to fly efficiently.
\end{itemize}

\textbf{Analysis:} While ``stay aloft'' accurately targets the fundamental aerodynamic requirement of overcoming gravity, CP-MoE provides a valid, evolutionarily consistent semantic variation (``fly efficiently'') embedded in a grammatically complete sentence. Since the prompt instruction does not limit the output to isolated phrases, CP-MoE's response is perfectly aligned with the task. However, the metric fails here on two fronts: it cannot recognise valid semantic paraphrasing, and it mathematically penalises the precision score due to the length of CP-MoE's conversational format. The LoRA branch achieves a higher F1 score strictly because its truncated answer (``to fly'') shares the preposition ``to'' and has a smaller denominator, not because it demonstrates superior knowledge retention.

\paragraph{Conclusion}
These cases confirm that CP-MoE effectively preserves high-resolution continuous knowledge. The apparent degradation in metrics is an artefact of the evaluation method's inability to assess semantic equivalence and its mathematical sensitivity to sequence length and exact token overlap.

\subsection{Parameter Efficiency and Computational Cost}
\label{subsec:efficiency}

As a PEFT-based framework, CP-MoE is designed to balance superior resistance to catastrophic forgetting with high computational efficiency. Its modular architecture allows for flexible configurations: a comprehensive Full-projection setup for complex instructions (e.g., SuperNI) and an optimised FFN-only setup for domain-specific tasks (e.g., VQA v2). 

The trainable parameters of CP-MoE consist of three core components: the split-expert LoRA branch ($P_{\mathrm{SE}}$), the transient expert branch ($P_{\mathrm{TE}}$), and the routing logic ($P_{\mathrm{Router}}$). For a target linear module with input dimension $in$ and output dimension $out$, given a total LoRA rank $r$ and $E$ experts (where $r_e = r/E$ is the rank per expert), the parameter count for each component is:
\begin{align}
    P_{\mathrm{SE}}(in, out) &= E \cdot (in \cdot r_e + r_e \cdot out) = r(in + out) \\
    P_{\mathrm{TE}}(in, out) &= r_e(in + out) \\
    P_{\mathrm{Router}}(in) &= in \times E
\end{align}

Specifically, for our SuperNI configuration on Llama-2-7B ($L=32, d=4096, m=11008$), we apply CP-MoE to all seven projection layers ($q, k, v, o, gate, up, down$) with $r=32$ and $E=8$ (rank 4 per expert). Substituting these values, the total trainable parameters $P_{\mathrm{total}}$ are derived as:
\begin{equation}
\begin{aligned}
P_{\mathrm{total}} = L \cdot \Big[ & 4 \cdot \big(P_{\mathrm{SE}}(d, d) + P_{\mathrm{TE}}(d, d) + P_{\mathrm{Router}}(d)\big) \\
& + 2 \cdot \big(P_{\mathrm{SE}}(d, m) + P_{\mathrm{TE}}(d, m) + P_{\mathrm{Router}}(d)\big) \\
& + \big(P_{\mathrm{SE}}(m, d) + P_{\mathrm{TE}}(m, d) + P_{\mathrm{Router}}(m)\big) \Big]
\end{aligned}
\end{equation}
yielding $P_{\mathrm{total}} = 99,057,664 \approx 99.06\mathrm{M}$, which represents $1.48\%$ of the backbone parameters. For the VQA v2 task, we employ an FFN-only configuration ($r=16, E=4$), resulting in $31.46\mathrm{M}$ trainable parameters ($0.47\%$).

    Notably, as shown in Table~\ref{tab:time_cost}, the trainable parameter count for CL-MoE is identical to that of LoRA-MoE . This is because the task-level router in the official implementation of CL-MoE functions as a statistical heuristic for counting expert activations rather than a trainable neural layer. Our experimental setup strictly follows this official implementation protocol.
    
    Table~\ref{tab:time_cost} summarises the computational overhead measured on two NVIDIA H200 GPUs. While the transient expert introduces slight computational costs, the overhead remains within a manageable range. For instance, in VQA v2, CP-MoE incurs approximately $4\%$ additional training time compared to CL-MoE while providing significantly better knowledge retention.

% --- TABLE 2: TIME & COST ---
\begin{table}[!htbp]
\centering
\small
\caption{\textbf{Computational Overhead and Parameter Efficiency.} Measured on two NVIDIA H200 GPUs. \textbf{\%  PEFT} denotes the proportion of total PEFT parameters relative to the full model.}
\label{tab:time_cost}
\resizebox{\columnwidth}{!}{%
\begin{tabular}{llccc}
\toprule
\textbf{Dataset} & \textbf{Method} & \textbf{Train Time / Epoch} & \textbf{PEFT Params} & \textbf{\% PEFT} \\
\midrule
\multirow{3}{*}{\textbf{SuperNI}} 
 & GainLoRA-infLoRA &  96.47 min & 46.83 M & 0.69\% \\
 & GainLoRA-O-LoRA  & 74.21 min & 83.26 M & 1.24\% \\
 & \cellcolor{gray!15}\textbf{CP-MoE (Ours)} & \cellcolor{gray!15}114.35 min & \cellcolor{gray!15}99.06 M & \cellcolor{gray!15}1.48\% \\
\midrule
\multirow{3}{*}{\textbf{VQA v2}} 
 & LoRA-MoE & 635.12 min & 25.66 M & 0.38\% \\
 & CL-MoE   & 628.89 min & 25.66 M & 0.38\% \\
 & \cellcolor{gray!15}\textbf{CP-MoE (Ours)} & \cellcolor{gray!15}657.43 min & \cellcolor{gray!15}31.46 M & \cellcolor{gray!15}0.47\% \\
\bottomrule
\end{tabular}%
}
\end{table}

% --- TABLE 2: SUPERNI MAIN TASKS DETAILS ---
\begin{table}[ht]
\centering
\caption{\textbf{Continual learning performance on SuperNI Main Tasks (1-8) Order 1.}}
\label{tab:main_tasks}
\resizebox{\columnwidth}{!}{%
\begin{tabular}{l cccccccccc}
\toprule
Task ID & 1572 & 363 & 1290 & 181 & 002 & 1510 & 639 & 1729 & \textbf{ACC} & \textbf{AF} \\ 
\midrule
CP-MoE & 32.93 & \textbf{88.00} & \textbf{28.66} & \textbf{61.77} & \textbf{71.64} & 98.25 & \textbf{8.90} & 16.56 & \textbf{50.84} & 0.62 \\
GainLoRA-infolora & 37.89 & 85.00 & 17.15 & 34.5 & 67.38 & 98.75 & 8.53 & 15.34 & 45.57 & \textbf{-0.28} \\
GainLoRA-olora & \textbf{42.65} & \textbf{88.00} & 26.22 & 52.38 & 62.43 & \textbf{99.04} & 8.35 & \textbf{17.73} & 49.60 & 0.82 \\ 
\bottomrule
\end{tabular}%
}
\end{table}

\subsection{Detailed Per-Task Performance on SuperNI Order 1}
\label{subsec:Detailed}
% --- TABLE 3: SUPERNI ZERO-SHOT DETAILS ---
\begin{table*}[!htbp]
\centering
\caption{\textbf{Zero-shot Transfer performance on SuperNI Tasks (9-15) Order 1.}}
\label{tab:zeroshot_tasks}
\resizebox{\textwidth}{!}{%
\begin{tabular}{l cccccccc}
\toprule
\textbf{Method} & Task 073 & Task 1590 & Task 748 & Task 511 & Task 591 & Task 1687 & Task 875 & \textbf{AVG} \\ 
\midrule
CP-MoE & \textbf{42.00} & 10.31 & \textbf{34.78} & \textbf{16.83} & 29.61 & \textbf{70.00} & \textbf{47.00} & \textbf{35.80} \\
GainLoRA-infolora & 24.93 & \textbf{11.35} & 34.08 & 11.95 & 36.55 & 37.18 & 38.33 & 27.77 \\
GainLoRA-olora & 35.00 & 12.33 & 27.68 & 14.52 & \textbf{47.44} & 55.00 & 44.67 & 33.80 \\ 
\bottomrule
\end{tabular}%
}
\end{table*}

Tables \ref{tab:main_tasks} and \ref{tab:zeroshot_tasks} provide the detailed per-task breakdown for the SuperNI benchmark under Order 1.

In the main continual learning sequence (Table \ref{tab:main_tasks}), CP-MoE demonstrates distinct advantages on specific tasks. Notably, on Task 181 and Task 002, it achieves scores of 61.77 and 71.64 respectively, substantially outperforming the GainLoRA variants. Across the remaining main tasks, CP-MoE maintains highly competitive performance, yielding the highest overall Average Performance (ACC) of 50.84 and a robust Average Forgetting (AF) of 0.62.

CP-MoE also demonstrates superior zero-shot generalisation (Table \ref{tab:zeroshot_tasks}). On unseen domains, CP-MoE exhibits significant performance spikes, particularly on Task 073 (42.00) and Task 1687 (70.00), where it outperforms the strongest baseline by large margins. For the rest of the zero-shot sequence, CP-MoE remains consistently competitive, resulting in a state-of-the-art average transfer score of 35.80.

\subsection{Ablation study for VQAV2 benchmark}
\label{subsec:Ablation for vaq}
\begin{table*}[!htbp]
\centering
\caption{Ablation study of the CP-MoE on VQAV2.}
\label{tab:ablation}
\resizebox{\textwidth}{!}{%
\begin{tabular}{@{}cc@{\hskip 0.2in}ccccccccccc@{}}
\toprule
\multicolumn{2}{@{}c@{\hskip 0.2in}}{Modules} & \multicolumn{11}{c}{Metrics} \\
\cmidrule(r){1-2} \cmidrule(l){3-13}
\begin{tabular}[c]{@{}c@{}}CP\\ Bias\end{tabular} & 
\begin{tabular}[c]{@{}c@{}}TE\\ Reg.\end{tabular} & 
Rec. & Loc. & Jud. & Com. & Cou. & Act. & Col. & Typ. & Sub. & Cau. & AVG \\
\midrule
-- & -- & 57.96 & 43.45 & 78.96 & 74.57 & 45.96 & 73.07 & 66.95 & 56.15 & 58.67 & 28.57 & 58.43 \\
-- & $\checkmark$ & 56.31 & 41.17 & 80.85 & 76.53 & 50.74 & 75.69 & 73.93 & 62.76 & 61.85 & 28.11 & 60.79 \\
$\checkmark$ & $\checkmark$ & \textbf{57.96} & \textbf{43.73} & \textbf{82.52} & \textbf{79.05} & \textbf{52.87} & \textbf{77.21} & \textbf{74.79} & \textbf{64.16} & \textbf{62.41} & \textbf{29.95} & \textbf{62.30} \\
\bottomrule
\end{tabular}%
}
\end{table*}

As shown in Table \ref{tab:ablation}, an ablation study on the core components of CP-MoE was conducted on the VQAV2 dataset to evaluate the impact of CP Bias and TE Regularisation. The baseline model without any additional modules achieves an average accuracy of 58.43\%. Introducing only TE Regularisation increases the average accuracy to 60.79\%, with distinct improvements in metrics such as Col. (+6.98\%) and Typ. (+6.61\%). Integrating both the CP Bias and TE Regularisation modules brings the model to optimal performance across all 10 evaluation metrics, resulting in a peak average accuracy of 62.30\%. This confirms that CP Bias works jointly with the regularisation module to improve logical judgment and reasoning capabilities.

\paragraph{Feature Space Isolation.} 
We further extract the output representations of the experts and project them into a 2D
space using t-SNE, as shown in Figure \ref{fig:tsne}. The baseline MoE exhibits severe feature entanglement, where representations from previous and current tasks overlap significantly, explaining the catastrophic forgetting. CP-MoE exhibits
significantly reduced overlap among expert activation regions, with several experts forming more compact and visually
separated clusters. {Moreover, as shown in Figure~\ref{fig:tsne_long_seq}, this separation persists even under a long task sequence (Tasks 11 and 15), demonstrating that CP-MoE maintains stable and well-separated expert representations as the sequence grows.}

% --- FIGURE: TSNE ---
\begin{figure}[!htbp]
\centering
\includegraphics[width=0.45\linewidth]{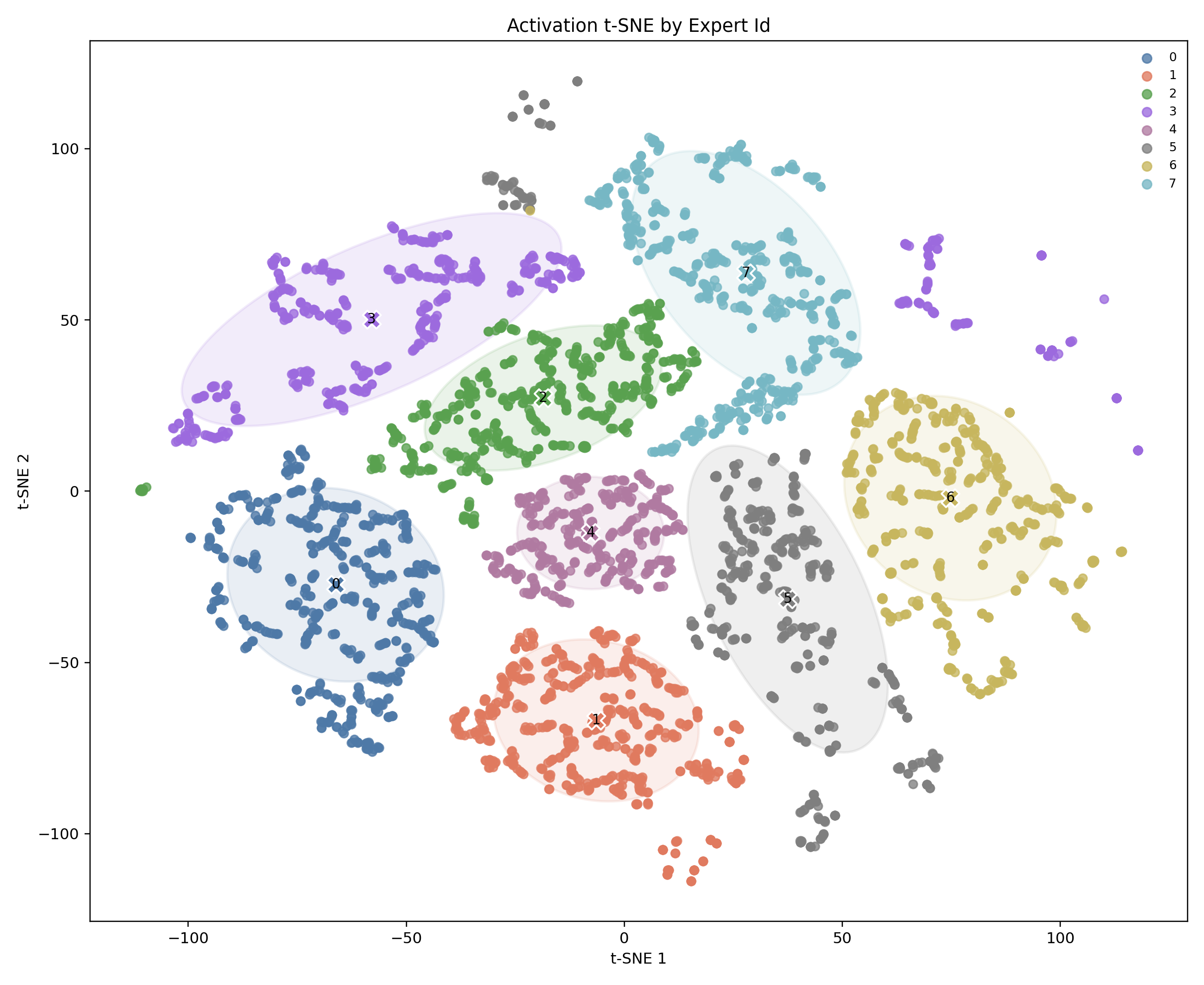} \hfill
\includegraphics[width=0.45\linewidth]{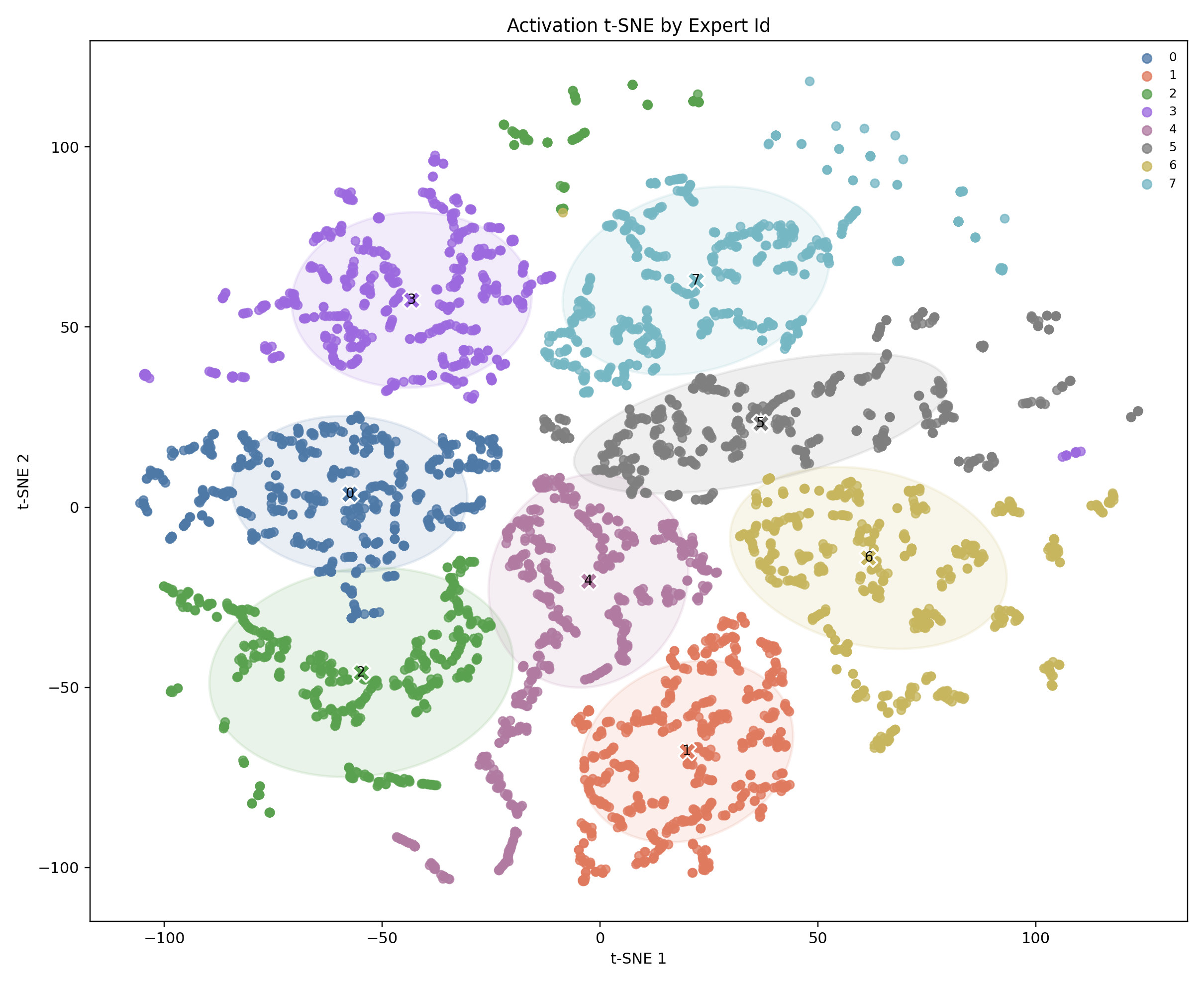}
\caption{\textbf{t-SNE Visualisation of Expert Representations.} \textbf{Left:} CP-MoE maintains clear geometric boundaries with more compact and separated clusters, preventing semantic interference. \textbf{Right:} The LoRA-MoE baseline exhibits severe feature entanglement and overlapping representations.}
\label{fig:tsne}
\end{figure}

    \begin{figure}[!htbp]
\centering
\setlength{\tabcolsep}{2pt}
\renewcommand{\arraystretch}{1.1}
\begin{tabular}{@{}cc@{}}
\small\textbf{CP-MoE (Ours)} & \small\textbf{LoRA-MoE} \\
\includegraphics[width=0.42\textwidth]{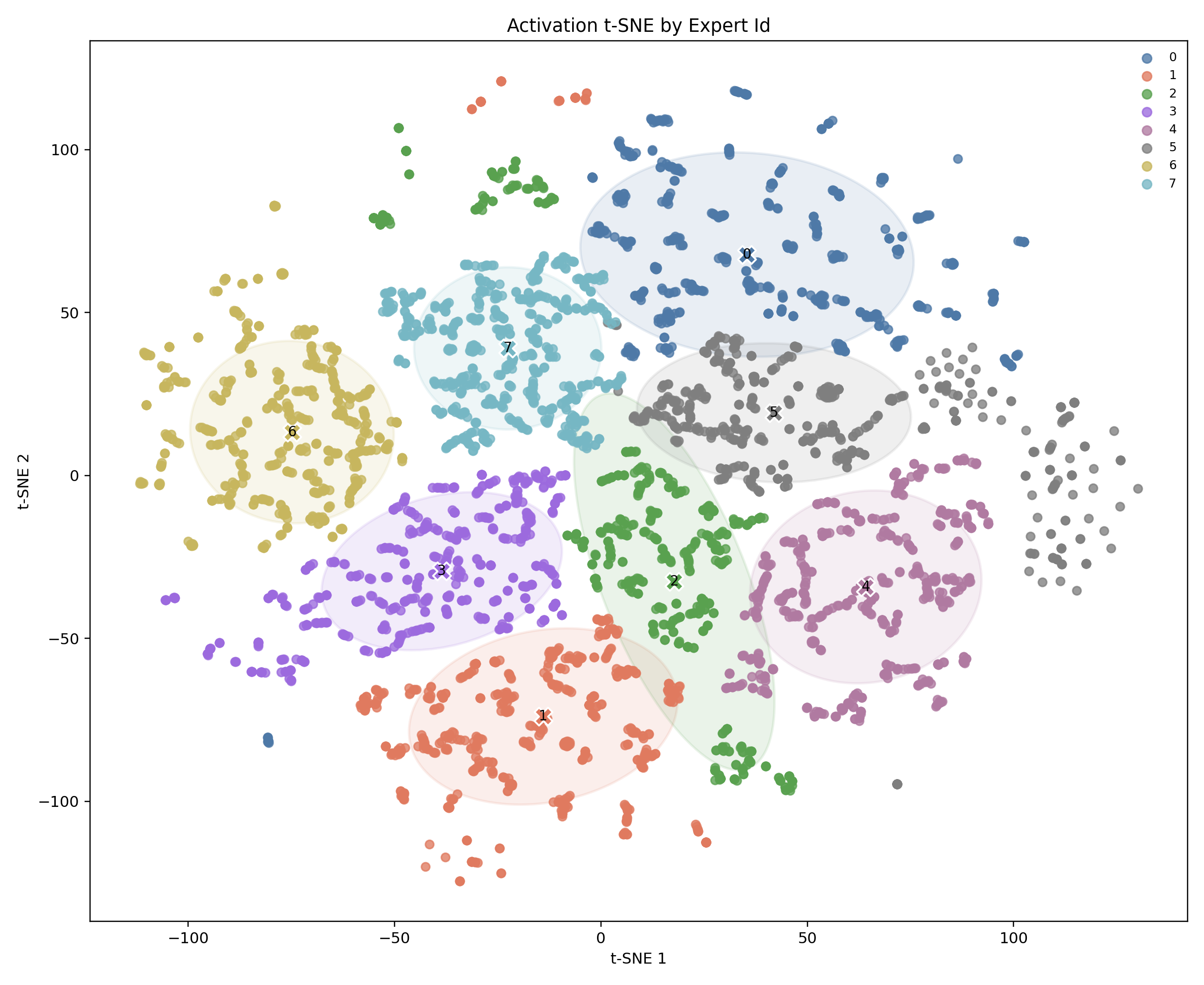} &
\includegraphics[width=0.42\textwidth]{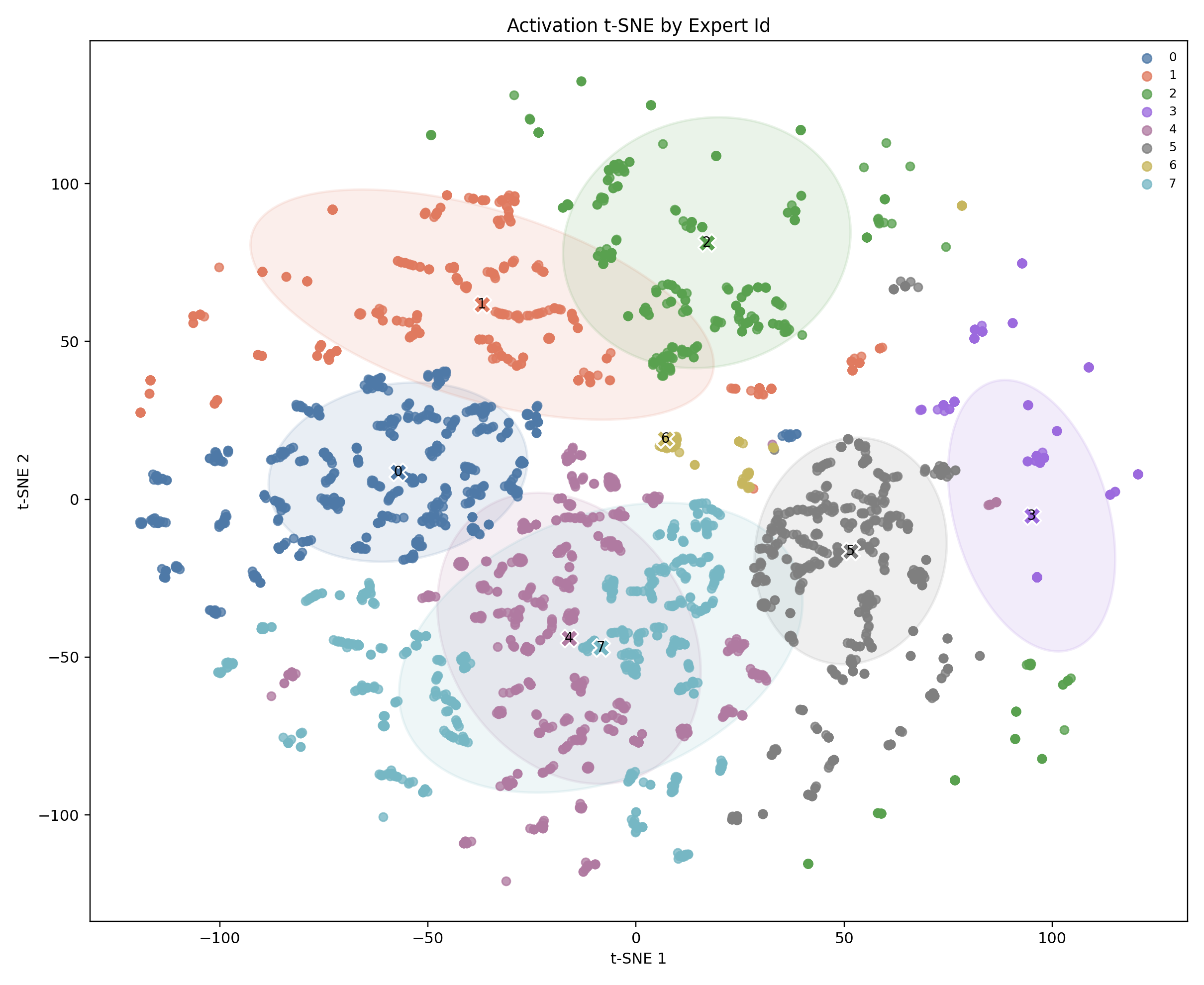} \\
\multicolumn{2}{c}{\small (a) Task 11} \\[4pt]
\includegraphics[width=0.42\textwidth]{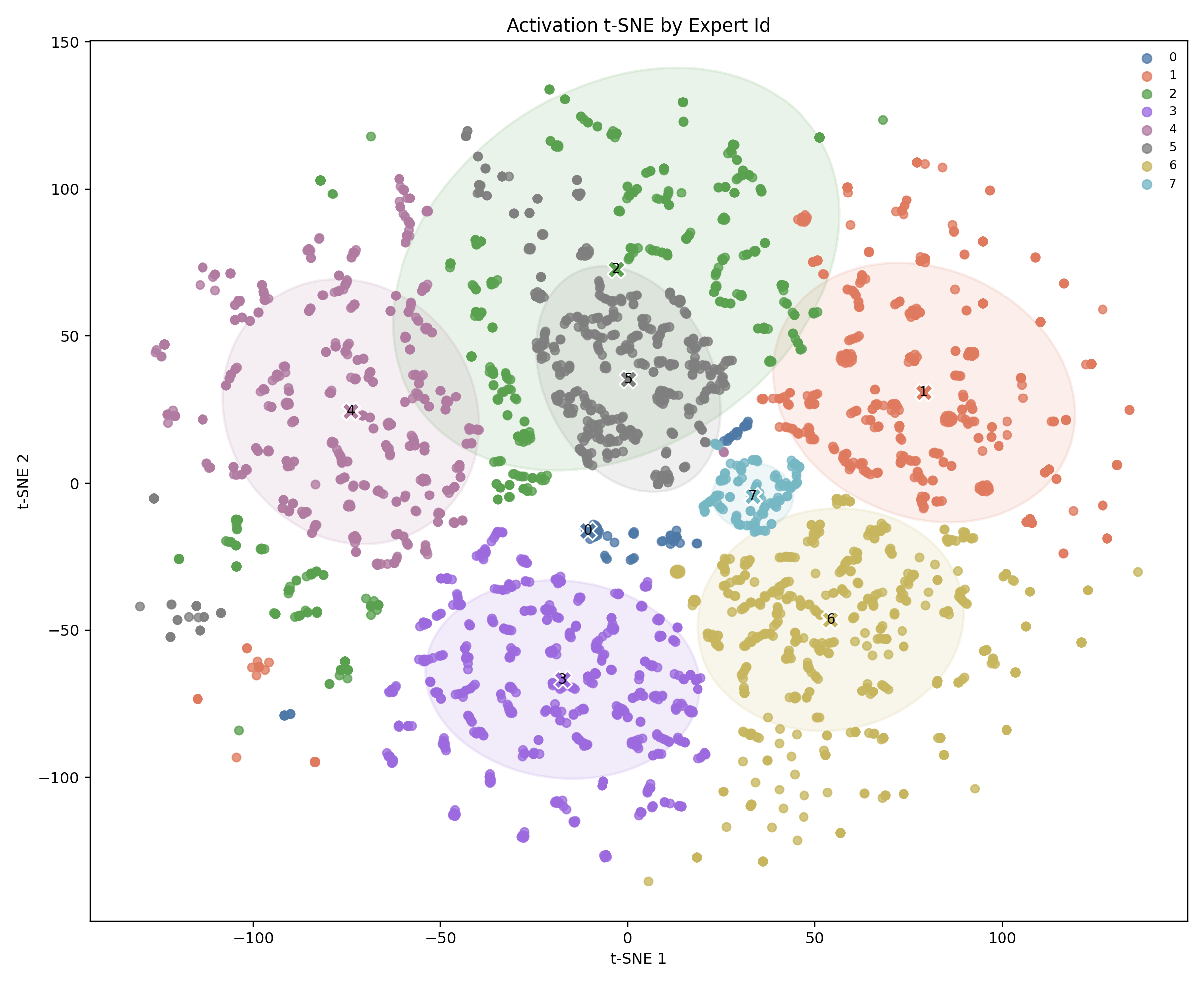} &
\includegraphics[width=0.42\textwidth]{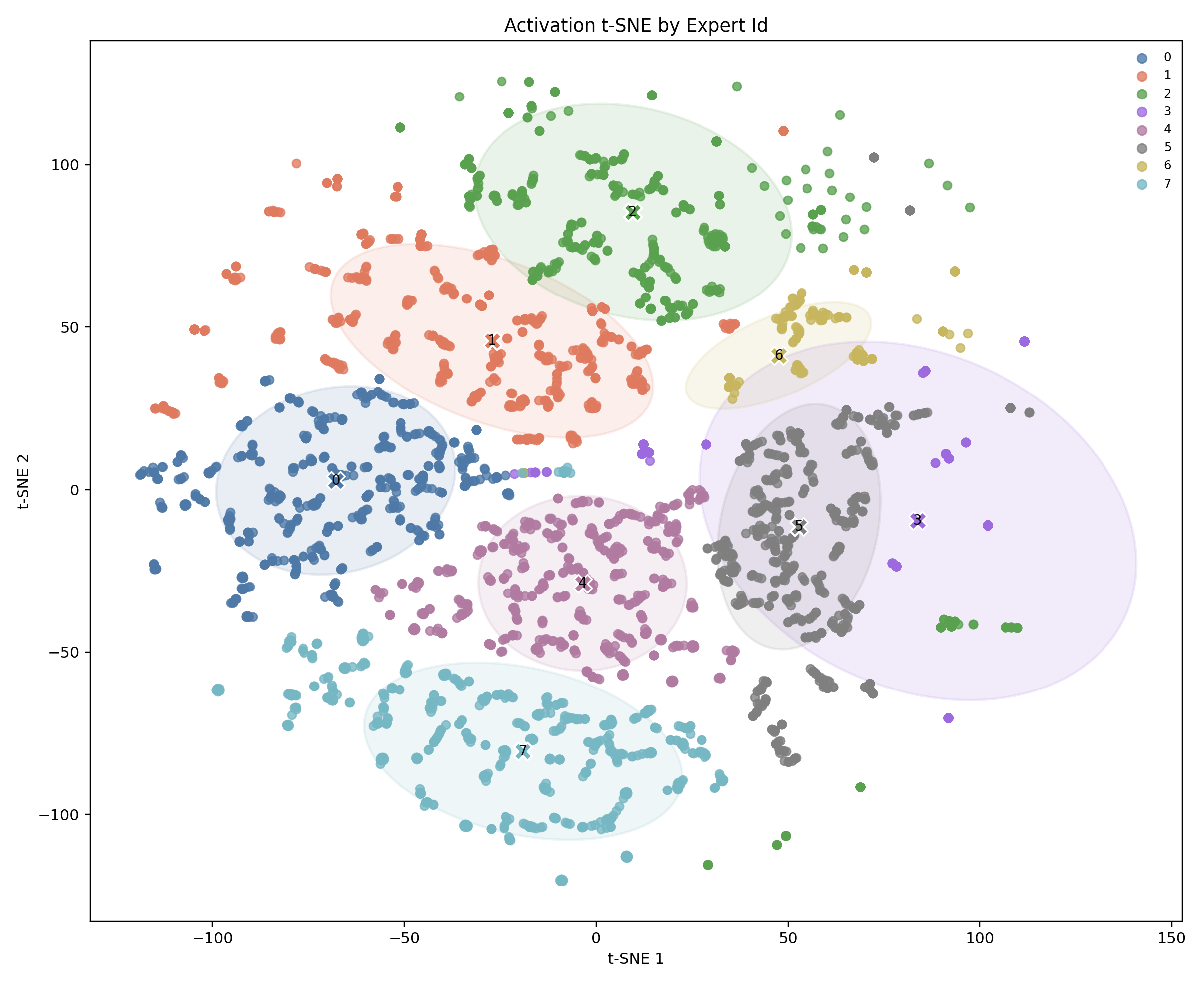} \\
\multicolumn{2}{c}{\small (b) Task 15} \\
\end{tabular}
\caption{\textbf{t-SNE of expert representations under a long task sequence.}
CP-MoE (left) vs.\ LoRA-MoE (right) at Task 11 (a) and Task 15 (b). Even as
the task sequence grows, CP-MoE maintains compact, well-separated expert
clusters, whereas LoRA-MoE exhibits persistent feature entanglement.}
\label{fig:tsne_long_seq}
\end{figure}

\paragraph{Aggregated Load Balancing in Continual Learning.} 
Figure \ref{fig:lb} illustrates the aggregated expert load across all layers (0-31) and experts (0-7) for CP-MoE after completing the continual learning task sequence. The heatmap demonstrates that CP-MoE achieves effective computational load distribution, maintaining reasonable activation proportions across experts. When evaluated alongside the high Average Performance and near-zero Average Forgetting, this distributed activation confirms that CP-MoE effectively balances system load while preserving historical task parameters to prevent catastrophic forgetting.

\begin{figure}[!htbp]
    \centering
    \includegraphics[width=0.75 \linewidth]{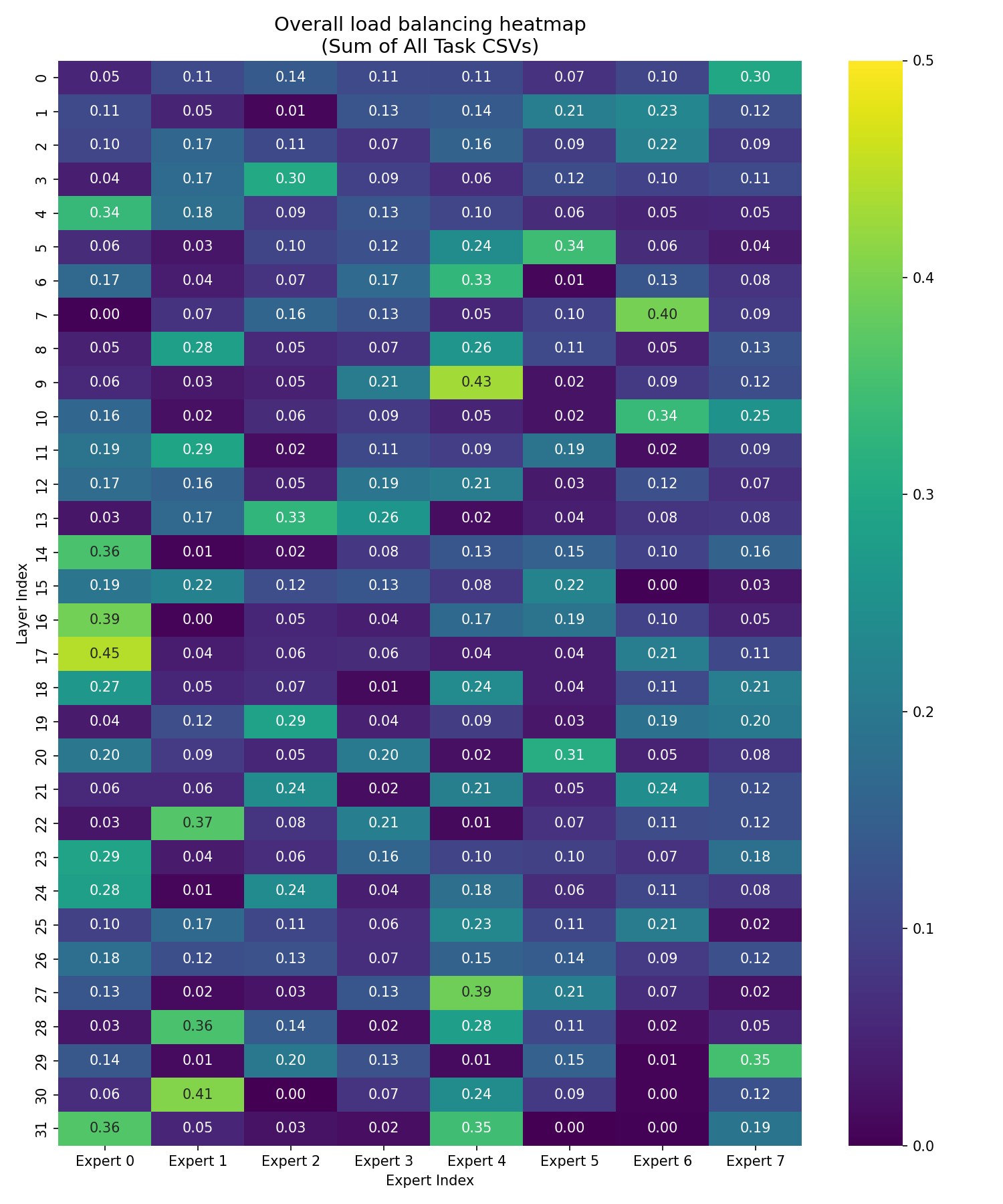}
    \caption{Overall expert load for CP-MoE}
    \label{fig:lb}
\end{figure}

\section{Algorithm}
\label{appendix:algorithm}

For clarity, we summarise the full CP-MoE procedure for a single task $t$ in Algorithm~\ref{alg:cpmoe}. The procedure consists of three stages. Stage~1 trains a transient expert on a warm-up subset $\hat{D}_t$ and uses its optimisation trajectory to estimate the prospective importance mask $\Omega_t$ via the path-integral rule (Eqs.~5--7). Stage~2 computes the CKA-based consistency scores $h^{CP}_i$ between the transient expert and each stable expert, and accumulates the protection-weighted importance into $\Omega^{(i)}_{\text{total}}$ (Eqs.~9, 13). Stage~3 then trains the stable experts on the full task data with CP-bias routing (Eqs.~10--11) and importance-weighted regularisation (Eqs.~14--15). The transient expert is discarded after Stage~2 and incurs no persistent computational overhead at inference.

\begin{algorithm}[H]
\caption{CP-MoE: Continual Learning for Task $t$}
% \caption{CP-MoE: Continual Learning for Task $t$}
\label{alg:cpmoe}
\begin{algorithmic}[1]
\Require Task data $D_t$; frozen backbone $\Theta_{\text{frozen}}$; stable experts $\Phi = \{E_i = B_i A_i\}_{i=1}^{n}$ with router params $\{W_i\}_{i=1}^{n}$; accumulated importance $\{\Omega^{(i)}_{\text{total}}\}_{i=1}^{n}$; parameter snapshots $\{A_i^{\text{old}}, B_i^{\text{old}}\}_{i=1}^{n}$ from task $t{-}1$
\Ensure Updated $\Phi$ and $\{\Omega^{(i)}_{\text{total}}\}$
\Statex \textbf{Stage 1: Transient Expert Probing on warm-up subset $\hat{D}_t$}
\State Initialise $\phi^{TE}_t = \{A^{TE}_t, B^{TE}_t\}$ such that $E^{TE}_t(x) = 0$; \,\, $\omega_t \gets 0$
\For{$s = 0$ to $S{-}1$} \Comment{Eq.~5}
    \State $g_s \gets \nabla_\phi \hat{\mathcal{L}}^{\,t}_{\text{task}}(\phi^{TE}_{t,s})$ on $\hat{D}_t$ \quad with $(\Theta_{\text{frozen}}, \Phi)$ fixed
    \State $\phi^{TE}_{t,s+1} \gets \phi^{TE}_{t,s} - \eta\, g_s$; \quad $\Delta\phi_s \gets \phi^{TE}_{t,s+1} - \phi^{TE}_{t,s}$
    \State $\omega_t \gets \omega_t + (-g_s \odot \Delta\phi_s)$ \Comment{Eq.~6}
\EndFor
\State $\Omega_t \gets \omega_t \oslash \big((\phi^{TE}_{t,S} - \phi^{TE}_{t,0})^{\odot 2} + \xi\big)$ \Comment{Eq.~7}
\Statex \quad\, \textit{($\Omega_t = \{\Omega_{t,A}, \Omega_{t,B}\}$ for the two LoRA factors)}
\Statex \textbf{Stage 2: CKA Scoring \& Importance Accumulation}
\For{$i = 1$ to $n$}
    \State $h^{CP}_i \gets \mathrm{CKA}(Z^{TE}, Z^{SE}_i)$ \Comment{Eq.~9}
    \State $\Omega^{(i)}_{A,\text{total}} \gets \Omega^{(i)}_{A,\text{total}} + h^{CP}_i \, \Omega_{t,A}$; \,\, $\Omega^{(i)}_{B,\text{total}} \gets \Omega^{(i)}_{B,\text{total}} + h^{CP}_i \, \Omega_{t,B}$ \Comment{Eq.~13}
\EndFor
\Statex \textbf{Stage 3: Stable Update with CP-Bias Routing and Protection}
\For{each minibatch $x \in D_t$}
    \State $\tilde{s}_i \gets x^\top W_i + \alpha\, h^{CP}_i, \;\; \forall i$ \Comment{Eq.~10}
    \State $\mathcal{K}_t \gets \mathrm{Top\text{-}}k\{\tilde{s}_i\}$; \,\, ${G_i} = \mathrm{softmax}_{i\in\mathcal{K}_t}(\tilde{s}_i)$ \Comment{Eq.~11}
    \State $\mathcal{L}^{\,t}_{\text{total}} \gets \mathcal{L}^{\,t}_{\text{task}} + \lambda\, \mathcal{L}^{\,t}_{\text{reg}}\big(\{\Omega^{(i)}_{\text{total}}\}, \Phi, \{A_i^{\text{old}}, B_i^{\text{old}}\}\big) + \gamma\, \mathcal{L}^{\,t}_{\text{aux}}$ \Comment{Eqs.~14--15}
    \State $\Phi \gets \Phi - \eta_\Phi \nabla_\Phi \mathcal{L}^{\,t}_{\text{total}}$
\EndFor
\State Snapshot $\{A_i^{\text{old}}, B_i^{\text{old}}\} \gets \{A_i, B_i\}$ \Comment{for next task}
\end{algorithmic}
\end{algorithm}

\section{Limitations and Future Work}
\label{sec:limitations}

While CP-MoE demonstrates strong performance in mitigating catastrophic forgetting, we acknowledge a key limitation regarding the sampling scale ($N$) of the Transient Expert (TE). In our current implementation, we empirically fix the warm-up sampling size across all tasks to maintain computational efficiency. However, preliminary observations indicate that the optimal sampling scale is highly sensitive to the dataset's intrinsic complexity and modality. For instance, tasks with high variance in their semantic distribution may require a larger $N$ to accurately capture the geometric manifold, whereas simpler tasks could converge with significantly fewer tokens. In future work, we plan to design an \textit{adaptive sampling mechanism} that dynamically determines the optimal $N$ based on real-time gradient variance, further optimising the trade-off between representation accuracy and computational overhead.
\end{document}